\documentclass{article}
\usepackage{lgno_arxiv}

\usepackage{amsthm}
\usepackage{mathrsfs}
\usepackage{algorithm}     \usepackage{algorithmic}   \usepackage{bm}
\usepackage{booktabs}
\usepackage{enumitem}
\usepackage{float}         \usepackage{framed,multirow}
\usepackage{comment}
\usepackage[dvipsnames,table]{xcolor}
\usepackage{latexsym}
\usepackage{amsmath}       \usepackage{amssymb}       \usepackage{bm}            \usepackage{amsmath,bm}
\renewcommand{\mathbf}{\bm}
\usepackage{algorithm}     \usepackage{algorithmic}   \usepackage{caption}       \usepackage{subcaption}    \usepackage{float}         \usepackage{graphicx}      \usepackage{comment}
\usepackage{epsfig}
\usepackage[numbers]{natbib}
\usepackage{booktabs}
\usepackage{multirow}
\usepackage{makecell}
\usepackage{siunitx}
\usepackage{adjustbox}
\usepackage{placeins}
\usepackage{makecell}
\usepackage{tabularx}
\newtheorem{theorem}{Theorem}[section]

\theoremstyle{definition}

\usepackage{mathtools}
\graphicspath{graph}
\usepackage{color} 
\usepackage{lineno}
\usepackage{multirow}
\usepackage{algorithm}
\usepackage{algorithmic}
\usepackage{amsmath}
\usepackage{psfrag}
\usepackage{graphicx}
\usepackage{mathrsfs}
\usepackage{booktabs}
\usepackage{diagbox}
\usepackage{xurl}
\usepackage{hyperref}
\definecolor{CMAMELink}{RGB}{0,112,128}
\hypersetup{
    colorlinks=true,
    linkcolor=CMAMELink,
    citecolor=CMAMELink,
    urlcolor=CMAMELink
}

\usepackage{cleveref}
\newcommand{\figref}[1]{\hyperref[#1]{Fig.~\ref*{#1}}}

\begin{document}

\title{Local Gradient Neural Operator}

\author{Baiming Zhang \\
Zhejiang University \\
\texttt{baimingzhang@zju.edu.cn}
\And
Jinsong Tang \\
Nanjing University of Science and Technology \\
\texttt{tangjs@njust.edu.cn}
\And
Ying Xu \\
Hefei University of Technology \\
\texttt{yingxu@hfut.edu.cn}
\And
Lihua Chen \\
Zhejiang University \\
\texttt{mecclh@zju.edu.cn}
\And
Shiying Xiong \\
Zhejiang University \\
\texttt{shiying.xiong@zju.edu.cn}}
\date{}
\maketitle

\begin{abstract}

Field temporal prediction and source identification constitute canonical problems in dynamical systems. Conventional approaches to these problems depend on a thorough understanding of the governing partial differential equations (PDEs). Recently, deep learning, as represented by neural operators, has provided a data-driven paradigm for addressing such tasks.
However, most existing global neural operators for PDEs require large training datasets and many learnable parameters, with limited interpretability and generalization. We propose the local gradient neural operator (LGNO) as a lightweight and interpretable alternative for field temporal evolution prediction and source identification in typical mechanical problems. The method builds on priors from nonlinear gradient discretization and uses multilayer perceptron convolutional layers to learn translation-invariant local kernels that resemble discrete stencils. A zero consistent stencil factorization separates coefficient learning from field reconstruction, rendering the learned operators more transparent. For problems with symmetries, network folding shares equivalent components and reduces parameter counts. We evaluate the method on PDE benchmarks covering linear and nonlinear, static and dynamic, and low and high dimensional cases. Results show that LGNO maintains accuracy, parameter efficiency, and rollout stability across these tasks, and further exhibits wide applicability to mechanical problems including diffusion, flow, and quantum phenomena.
\end{abstract}


\section{Introduction}
\label{sec:Introduction}

Partial differential equations (PDEs)~\cite{JIANG2026114525,HUZAIFAYASEEN2025105650,ZHANG2025431,RAJPUT2025101316,SCHABERGER2026100402,DIFONZO2026102945} characterize diverse physical phenomena such as diffusion~\cite{ZHANG2026110205,SHIRAKAWA2026115006,GUO2026115143,WANG2026111203,HAO2025110108}, solid deformation~\cite{LI2026114681,ZENG2026111879,QIN2025104378,TANG2024111853,KUNDU2022107102}, fluid motion~\cite{CHEN2026110992,XUE2026204593,yanruquaternionic,DONG2026114837,huang2025electrokinetic,gong2025dynamic}, and quantum dynamics~\cite{ACHARYA2026135092,VALES2026114992,PUZZUOLI2023112262}. Two primary tasks arise in scientific and engineering applications. The first is field temporal evolution prediction~\cite{men2025multiple,eshaghi2026multi,li2026dynamic}, which extrapolates subsequent field distributions given instantaneous field data. The second is source identification~\cite{paixao2026hybrid,TANG2026114497}, which reconstructs unknown external excitations from measured field observations. Conventional numerical methods, including finite difference~\cite{SANCHEZGALVIS2026106218,VARGAS2026108337,xu2026gql}, finite element~\cite{PARK2026108014,TANG2023106920,jiang2024dynamic,CHOI2024107177}, and finite volume schemes~\cite{ZUO2026115147,LI2026114817,flint2026eulerian}, require complete knowledge of PDE governing equations, coefficients, and discretization configurations, which are often inaccessible for black box engineering systems with unknown governing laws~\cite{hou2026efficient,xu2026high,WANG2025109783}.

Deep learning has emerged as a data driven alternative for solving PDE-based problems~\cite{zhang2018deep,liu2024multi,park2025unifying,tang2025quantum,sun2026novel,zhu2025continuous,zhang2025scientometric}. Among such methods, physics-informed neural networks (PINNs)~\cite{karniadakis2021physics,raissi2019physics,raissi2020hidden,xiong2025deep,zhao2025physics} incorporate PDE residual loss into training objectives and perform well in solving forward and inverse problems when the differential formula is fully known. PINNs require the PDE residual form, however, which limits their applicability when the PDE is unknown, and the pointwise fitting approach may present optimization challenges even when the PDE is known~\cite{lu2021deepxde,wang2022and,li2024physics}.

Neural operator learning provides a framework for learning mappings between function spaces directly from data, without requiring the explicit PDE form~\cite{wang2026aiPDEreview,wang2026separated,zou2025uncertainty,azizzadenesheli2024neural,lu2024bridging,kovachki2023neural}. Deep operator networks (DeepONets)~\cite{lu2021learning,sarkar2026learning,KIYANI2026113952,IVAGNES2026118900} utilize the branch-trunk architecture to approximate nonlinear PDE operators. Fourier neural operators (FNOs)~\cite{li2021fourier,liu2025ms,li2025d,yu2026inverse} and graph neural operators (GNOs)~\cite{li2020neural,sarkar2025spatio,vo2026attention,sarkar2026physics} further enrich this framework. These methods perform well when sufficient training data are available. In the limited-data regime, however, learning a full-domain mapping from few samples may lead to overfitting and make it difficult to recover the intrinsic local operator law of PDEs. For many PDE governed physical fields, the temporal change or source response at one spatial location depends primarily on the field values in a small neighborhood~\cite{peetre1959caracterisation,courant1928partiellen}. Global architectures also exhibit spectral bias toward smooth low-frequency components, which can affect their resolution of local gradient-sensitive structures~\cite{goswami2023physics,rabczuk2023machine}.

Motivated by the local discrete nature of differential operators, researchers have developed stencil-aware local neural architectures. Multilayer perceptron convolution (MLPConv)~\cite{lin2014network} replaces a linear convolutional filter with a small nonlinear micro-network. PDE-Net~\cite{long2017pde} uses convolutional filters to approximate differential stencils and combines the resulting local derivative features through learnable nonlinear mappings. Convolutional neural operators (CNOs)~\cite{NEURIPS2023_f3c1951b} further extend local convolutional representations through multi-scale convolution hierarchies. These methods show that local stencil structures can be effectively incorporated into neural operator learning. 
The local-solution-operator informed neural network (LOINN)~\cite{jiao2025one} model further adopts a one-pass forward modeling paradigm to infer physical field evolution from local neighborhood features, providing a lightweight solution for PDE solving under limited samples. Time-adaptive operator learning via neural Taylor expansion (TANTE)~\cite{wu2026tante} introduces a different form of locality by using neural Taylor expansion around the current state and adaptive time stepping to improve temporal rollout. A direct local window output, however, still does not explicitly emphasize the gradient-driven response of the underlying operator. 
Table~\ref{tab:method_comparison} summarizes the structural differences between representative methods and the proposed approach.

\begin{table}[htbp]
\centering
\caption{Structural comparison between the selected baselines and LGNO. Here, "operator form" distinguishes direct global mapping, local stencil mapping, and gradient-involved differential stencil reconstruction, while "model size" gives a qualitative indication of the typical parameter scale.}
\label{tab:method_comparison}
\renewcommand{\arraystretch}{1.15}
\setlength{\tabcolsep}{5pt}
\begin{tabular*}{\linewidth}{@{\extracolsep{\fill}}lllll}
\toprule
Method & Representation & Range & Operator form & Model size \\
\midrule
MLPConv~\cite{lin2014network} & Local MLP & Local & Local stencil map & Small \\
DeepONet~\cite{lu2021learning} & Branch-trunk & Global & Direct map & Large \\
LOINN~\cite{jiao2025one} & One-shot operator learning & Local & Direct map & Medium \\
\textbf{LGNO} & \textbf{Gradient-involved MLP} & \textbf{Local} & \textbf{Differential stencil} & \textbf{Small} \\
\bottomrule
\end{tabular*}
\end{table}

We propose a local gradient neural operator (LGNO) for black-box source and temporal prediction under limited training data. The method formulates both tasks as local mappings from neighborhood information to an operator response: for temporal prediction, the response is a time derivative or one-step update; for source prediction, the response is the unknown forcing or source term.

The LGNO design uses local gradient oriented reconstruction instead of direct pointwise output. A sliding window extracts the neighborhood pattern around each grid point. A lightweight micro-network generates local kernel coefficients from this neighborhood. The target response is reconstructed through a pseudo-linear combination of these coefficients. This structure resembles a learnable nonlinear stencil. It offers sensitivity to local gradients, gives the learned coefficients a discretization interpretation, and reduces zero point drift.

LGNO also exploits symmetry induced sample augmentation when the physical setting allows. In many PDE datasets, different field components, or coordinate directions may share the same operator law. Rather than feeding all symmetric components in parallel as independent channels, LGNO folds symmetry related samples into serial local training instances. This reduces the number of input channels and trainable parameters while increasing the effective number of local samples. This data expansion benefits the single sample or few sample regime.

The method is evaluated on source and temporal prediction benchmarks, including diffusion, Burgers type dynamics, Navier-Stokes flow, Gross-Pitaevskii dynamics, and Schr\"odinger-type evolution. Across these examples, LGNO operates as a black-box model without requiring the explicit PDE form or known physical coefficients.

The remainder of this paper is organized as follows. Section~\ref{sec:proposed_LGNO_framework} presents the proposed framework in detail. Section~\ref{sec:results} describes numerical experiments and results. Section~\ref{sec:conclusion} concludes with a summary and future directions. The code associated with this work is publicly available at \url{https://github.com/baiming-zhang/LGNO}.

\section{The proposed LGNO framework}
\label{sec:proposed_LGNO_framework}

In this section, we first present the theoretical background and overall architecture of LGNO, then discuss several of its advantageous properties, with the detailed mathematical proofs relegated to the appendix.

\subsection{Overview of the LGNO framework}
\label{sec:framework_overview}

Consider a $d$-dimensional bounded Lipschitz spatial domain $\Omega \subset \mathbb{R}^d$ over the time interval $[0,T]$. The general continuous PDE formulation can be written as
\begin{equation}
\label{eq:PDE formulation}
\partial_t \bm{u}(\bm{x},t) + \mathcal{L}\big(\bm{x},\,\bm{u},\,\nabla \bm{u},\,\nabla^2 \bm{u},\,\dots,\,\nabla^k \bm{u}\big) = \bm{f}(\bm{x},t),
\quad (\bm{x},t)\in \Omega \times [0,T]. 
\end{equation}
where $\bm{x} = (x_1,x_2,\dots,x_d)$ stands for the spatial coordinate vector, and $t\in [0,T]$ represents time with $0$ and $T$. $\Omega \times [0,T]$ defines the complete spatio-temporal solution domain for all physical quantities.
 $\bm{u}=\bm{u}(\bm{x},t): \Omega\times[0,T]\to\mathbb{R}^{c_u}$ is the primary physical field. $c_u$ denotes the number of independent field components, which is decoupled from spatial dimension $d$.  $\partial_t \bm{u}(\bm{x},t)$ is the first order time derivative of the field $\bm{u}$.
 $\mathcal{L}\big(\bm{x},\,\bm{u},\,\nabla \bm{u},\,\nabla^2 \bm{u},\,\dots,\,\nabla^k \bm{u}\big)$ is a generic $k$-th order spatial differential operator that encapsulates diffusion, advection, dispersion and nonlinear coupling effects. 
 $\bm{f}=\bm{f}(\bm{x},t)$ is the spatio-temporally distributed external excitation vector or potential field.

   The PDE residual representation $\mathcal{R}[\cdot;\cdot]$ can be further derived by rearranging all terms to one side of the equality, the original PDE is equivalent to a zero residual constraint:
\begin{equation}
\label{eq:PDE residual representation}
\mathcal{R}[\bm{u};\,\bm{f}](\bm{x},t) = \partial_t \bm{u}(\bm{x},t) + \mathcal{L}\big(\bm{x},\,\bm{u},\,\nabla \bm{u},\,\nabla^2 \bm{u},\,\dots,\,\nabla^k \bm{u}\big) - \bm{f}(\bm{x},t) = 0. 
\end{equation}
The semicolon in $\mathcal{R}[\bm{u};\,\bm{f}]$ distinguishes two independent input arguments: the primary field $\bm{u}$ and the source term $\bm{f}$.

Crucially, differential operators obey three basic properties: strict locality, compact stencil discretization and translation equivariance, which serve as the foundation of the local stencil structure used below.
Finite order differential operators are local in the sense that their evaluation at a point depends only on the field and finitely many of its derivatives in an arbitrarily small neighborhood of that point \cite{peetre1959caracterisation,peetre1960rectification}. Under standard finite difference discretization, this locality is represented by a finite neighboring stencil \cite{courant1928partiellen}. When the coefficients and the computational setting are spatially homogeneous, the resulting operator also commutes with translations, giving translation equivariance \cite{hormander1960translation}. In spatially heterogeneous settings, exact translation equivariance may be broken, but the operator response can still be locally characterized and approximated by a local neural network with varying stencil weights~\cite{jiao2025one}.

Based on the basic properties of PDE and the local stencil representation of discretized PDE (\ref{sec:discrete_kernel}), we propose LGNO to learn a state dependent local stencil rule and assemble the resulting pointwise responses over the computational domain. \figref{fig:flowchart} gives the workflow of the proposed LGNO framework considering the data-driven setting.

\begin{figure}[htb]
\centering
\includegraphics[width=\linewidth]{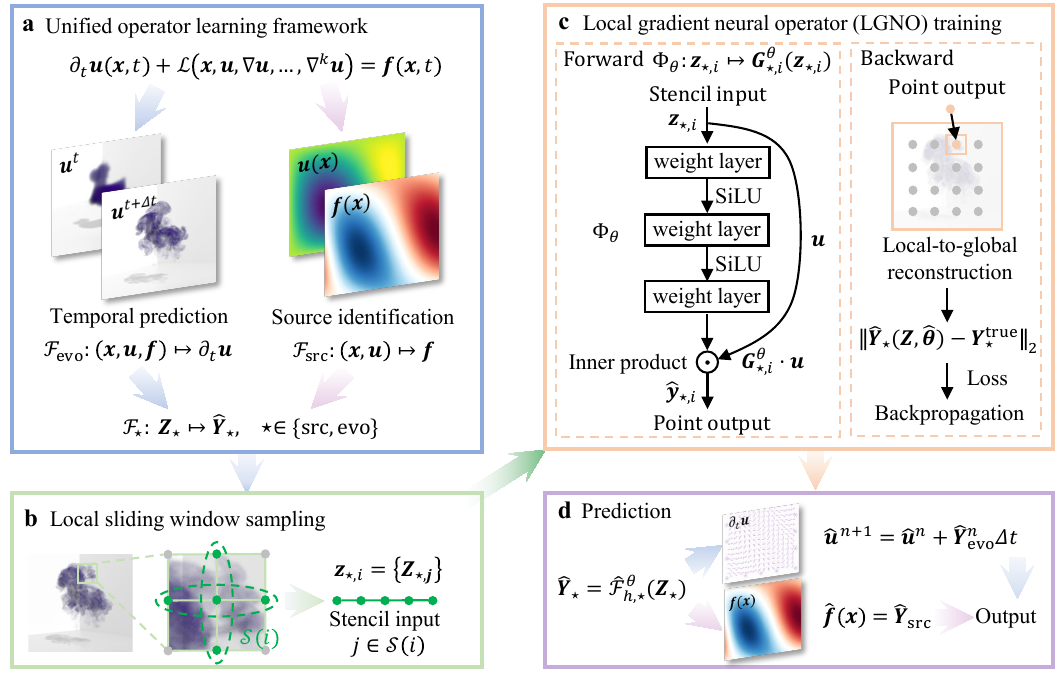}
\caption{Unified operator learning framework of the proposed LGNO. The model learns a local stencil-based operator from sliding-window samples and assembles the pointwise responses into a global prediction. Temporal evolution prediction and field-source prediction use the same local-to-global numerical structure, but the task-specific input and target fields have different physical meanings.}
\label{fig:flowchart}
\end{figure}

Two learning tasks are considered in this work, as shown in \figref{fig:flowchart}(a). They use the same local reconstruction form but assign different roles to the input and target fields. 

In field temporal evolution prediction, the primary field \(\bm{u}\) is given. The source related field \(\bm{f}\) is used as an external input. The target is the local time derivative \(\partial_t\bm{u}\). A general temporal evolution task $\mathcal{F}_{\mathrm{evo}}$ is written as
\begin{equation}
\label{eq:evolution_mapping_lgno}
\mathcal{F}_{\mathrm{evo}}:\bm{Z}_{\mathrm{evo}}=
(\bm{x}, \bm{u},\bm{f})
\mapsto
\widehat{\bm{Y}}_{\mathrm{evo}}=\partial_t\bm{u}.
\end{equation}

After predicting the time derivative, the field is advanced by a forward Euler update, following the standard time integration view used in neural ordinary differential equation formulations~\cite{chen2018neuralode,tang2025integrating}:
\begin{equation}
\label{eq:euler_rollout_lgno}
\widehat{\bm{u}}^{n+1}
=
\widehat{\bm{u}}^{n}
+
\widehat{\bm{Y}}_{\mathrm{evo}}^{n}\Delta t.
\end{equation}

In source identification, the observed primary field \(\bm{u}\) is the external input. The identification task $\mathcal{F}_{\mathrm{src}}$ of unknown source related field is written as:
\begin{equation}
\label{eq:source_mapping_lgno}
\mathcal{F}_{\mathrm{src}}:
\bm{Z}_{\mathrm{src}}=(\bm{x}, \bm{u})
\mapsto
\widehat{\bm{Y}}_{\mathrm{src}}=\bm{f}.
\end{equation}
Thus, \(\bm{f}\) appears as an input in field temporal evolution prediction when it is prescribed, but becomes the output target in source identification.

The two tasks of field temporal evolution prediction and source identification can be formulated under a unified notational system through task specific operator mappings:
\begin{equation}
\label{eq:unified_task_mapping}
\mathcal{F}_{\star}:\bm{Z}_{\star}\mapsto\widehat{\bm{Y}}_{\star},
\qquad
\star\in\{\mathrm{evo},\mathrm{src}\}.
\end{equation}
For each task, \(\bm{Z}_{\star}\) denotes the task specific input field and \(\widehat{\bm{Y}}_{\star}\) denotes the target field.

As indicated by this formulation, every single local window serves as one independent training sample, containing only the input field data within the stencil range and the target output value at the central grid point. Benefiting from the compact stencil discretization property of differential operators, the response of discrete operators at any grid node can be fully approximated using local window information rather than full domain field data.  The training pipeline is visualized in \figref{fig:flowchart}(b) and (c). LGNO extracts sliding windows from task inputs and predicts stencil-like weights, and local outputs are aggregated into global predictions.

Consider the corresponding task specific local input window which is defined as follows:
\begin{equation}\label{eq:window-def}
\mathbf{z}_{\star,i} = \big\{\bm{Z}_{\star,j}
:
j\in\mathcal{S}(i)\big\}.
\end{equation}
\noindent Based on Theorem~\ref{theorem:sparse}, the two tasks share the following stencil mapping form, giving the local output:
\begin{equation}
\label{eq:unified_stencil_form}
\widehat{\bm{y}}_{\star,i}
=
\widehat{\mathcal{F}}_{h,\star}^{\theta}(\mathbf{z}_{\star,i})
=
\sum_{j\in\mathcal{S}(i)}
\bm{G}_{\star,i,j}^{\theta}(\mathbf{z}_{\star,i}) \cdot \bm{u}_{i,j}
,
\qquad
\star\in\{\mathrm{evo},\mathrm{src}\}.
\end{equation}

Here \(\mathcal{S}(i)\) denotes the local stencil centered at grid point \(\bm{x}_i\).
\(\bm{G}_{\star,i,j}^{\theta}(\mathbf{z}_{\star,i})\) is the local matrix coefficient mapping the primary field at node \(j\) to the target response at node \(i\). Therefore, we train the shared MLPConv module to produce the state dependent stencil coefficients through this~\cite{lin2014network}:
\begin{equation}
\label{eq:coefficient_generator}
\Phi_{\boldsymbol{\theta}}
:
\mathbf{z}_{\star,i}
\mapsto
\left\{
\bm{G}_{\star,i,j}^{\theta}(\mathbf{z}_{\star,i})
\right\}_{j\in\mathcal{S}(i)}.
\end{equation}

Translation equivariance of the local architecture justifies the use of a shared local mapping across grid points. For any admissible translation vector \(\boldsymbol{\Delta}\), define the translation operator by
\begin{equation}
T_{\boldsymbol{\Delta}}\mathbf{z}_{\star}(\bm{x})
=
\mathbf{z}_{\star}(\bm{x}-\boldsymbol{\Delta}),\quad T_{\boldsymbol{\Delta}}\mathbf{y}_{\star}(\bm{x})
=
\mathbf{y}_{\star}(\bm{x}-\boldsymbol{\Delta}) .
\end{equation}
When the grid, boundary treatment, and calibrated operator are compatible with translations, the learned local operator is expected to satisfy: 
\begin{equation}
\label{eq:lgno_translation_equivariance}
\widehat{\mathcal{F}}_{h,\star}^{\theta}
\left[
T_{\boldsymbol{\Delta}}\mathbf{z}_{\star}(\bm{x})
\right]  
=
\widehat{\mathcal{F}}_{h,\star}^{\theta}
\left[
\mathbf{z}_{\star}(\bm{x}-\boldsymbol{\Delta})
\right]  
=
\widehat{\bm{y}}_{\star}(\bm{x}-\boldsymbol{\Delta})
=
T_{\boldsymbol{\Delta}}
\widehat{\bm{y}}_{\star}(\bm{x}) 
=
T_{\boldsymbol{\Delta}}
\widehat{\mathcal{F}}_{h,\star}^{\theta}
\left[
\mathbf{z}_{\star}(\bm{x})
\right]  .
\end{equation}
This property justifies using one shared local network to predict stencil coefficients at different grid points.

The global output is assembled from local predictions, in the same local-to-global sense used in standard numerical discretizations:
\begin{equation}
\label{eq:global_lgno_output}
\widehat{\mathbf{Y}}_{\star}
=\{\widehat{\mathbf{y}}_{\star,i}\}_{i\in\widehat{\Omega}}
=
\{
\widehat{\mathcal{F}}_{h,\star}^{\theta}(\mathbf{z}_{\star,i})
\}_{i\in\widehat{\Omega}},
\end{equation}
where \(\widehat{\Omega}\) is the set of grid nodes in the domain \(\Omega\) for which the prescribed stencil is well defined. Boundary nodes are treated according to the discretization used in each problem, e.g., via periodic wrapping.

This local sampling strategy increases the number of usable training samples without requiring many full field trajectories. Since the same differential law is applied at different spatial locations, local windows extracted from one or a few fields can provide repeated observations of the same underlying operator rule.

The trainable parameters \(\bm{\theta}\) are optimized by minimizing the discrepancy between the predicted responses \(\widehat{\mathbf{Y}}_{\star}(\mathbf{Z};\hat{\bm{\theta}})\) and the ground truth values \(\mathbf{Y}_{\star}^{\mathrm{true}}(\mathbf{x})\):
\begin{equation}
\label{eq:lgno_training_loss}
\bm{\theta}^{\ast}
=
\underset{\hat{\bm{\theta}}}{\arg\min}
\left\|
\widehat{\mathbf{Y}}_{\star}(\mathbf{Z};\hat{\bm{\theta}})
-
\mathbf{Y}_{\star}^{\mathrm{true}}(\mathbf{x})
\right\|_2.
\end{equation}

This parameter sharing enforces a common local rule and reduces the number of trainable parameters. After training, the coefficient generator together with the stencil contraction defines the learned discrete operator \(\widehat{\mathcal{F}}_{h,\star}^{\theta}\).

The main architectural distinction is the separation between coefficient prediction and target reconstruction. Instead of directly outputting the single point nonlinear target field as MLPConv does, LGNO first predicts local stencil coefficients and then constructs the response through an explicit contraction with neighboring primary field values, preserving the first-order gradient effects of the local numerical structure of PDE discretizations. This decomposition reduces the burden on the nonlinear network, keeps the final reconstruction tied to a discrete stencil form, and makes the learned operator closer to classical finite difference or finite volume representations.
This design embeds the locality of differential operators into the architecture rather than leaving it to be inferred from data alone, found useful in the tested small-sample regimes involving sharp gradients or strong nonlinearities.

\subsection{Zero-consistent reconstruction and error estimate}
\label{sec:zero_consistency_error}

LGNO uses a bias-free stencil contraction to enforce zero consistency on the calibrated local target. For many homogeneous physical operators, the zero-field state naturally gives a zero operator response, so no additional correction is needed. This is the case for the benchmark targets used in the numerical experiments below, where the learned responses are defined as calibrated differential or source responses and therefore vanish when the contracted primary-field channels are zero. In more general settings, prescribed coordinates, background potentials, constant source terms, or other conditioning channels may induce a finite zero-state response even when the contracted primary field vanishes. Such cases can be treated by subtracting the zero-state response before applying the bias-free reconstruction. The detailed derivation and proof are given in~\ref{sec:zero_consistency}. 

Let \(\mathbf{z}_{\star,i}^{0}\) denote the zero-state local window, where the physical field channels are set to zero while coordinate information and other prescribed conditioning channels, if present, are retained. For the raw local response, the zero-state offset is removed by

\begin{equation}
\label{eq:zero_state_calibration_main}
\mathcal{F}_{h,\star}
\left(
\mathbf{z}_{\star,i}
\right)
=
\mathcal{F}_{h,\star}^{\mathrm{raw}}
\left(
\mathbf{z}_{\star,i}
\right)
-
\mathcal{F}_{h,\star}^{\mathrm{raw}}
\left(
\mathbf{z}_{\star,i}^{0}
\right).
\end{equation}
The LGNO reconstruction is
\begin{equation}
\label{eq:zero_consistent_lgno_reconstruction}
\widehat{\bm{y}}_{\star,i}
=
\widehat{\mathcal{F}}_{h,\star}^{\theta}
\left(
\mathbf{z}_{\star,i}
\right)
=
\sum_{j\in\mathcal{S}(i)}
\bm{G}_{\star,i,j}^{\theta}
\left(
\mathbf{z}_{\star,i}
\right)
\bm{u}_{i,j}.
\end{equation}
Since the reconstruction contains no additive bias and the contracted primary field values vanish in the zero-state window, it follows directly that
\begin{equation}
\label{eq:zero_consistent_lgno_output}
\widehat{\mathcal{F}}_{h,\star}^{\theta}
\left(
\mathbf{z}_{\star,i}^{0}
\right)
=
\bm{0}.
\end{equation}

The resulting approximation error can be decomposed into a compact-stencil discretization error and a learned-coefficient error, following the standard separation between truncation error and approximation error in finite difference analysis~\cite{leveque2007finite}:
\begin{equation}
\label{eq:zero_consistent_error_decomposition}
\left\|
R_h
\mathcal{F}_{\star}
[
\bm{Z}_{\star}
]
-
\widehat{\mathcal{F}}_{h,\star}^{\theta}
[
\bm{Z}_{\star}
]
\right\|_{\ell_h^2}
\le
C_d h^p
+
C_{\star}
\varepsilon_{\theta}
\|\bm{u}\|_{\ell_h^2}.
\end{equation}
Thus, near the zero state, the bias-free contraction suppresses artificial output drift.

For temporal evolution prediction, a standard stability argument for time-dependent numerical schemes gives the corresponding rollout error estimate~\cite{birke2026error,qi2026efficient,paraschis2026linearly}:
\begin{equation}
\label{eq:zero_consistent_rollout_error}
\|\bm{e}^n\|_{\ell_h^2}
\le
e^{LT}
\|\bm{e}^0\|_{\ell_h^2}
+
\frac{e^{LT}-1}{L}
\zeta_{\theta}(h),
\qquad
T=n\Delta t.
\end{equation}
When \(L=0\), the second term is understood as \(T\zeta_{\theta}(h)\).

\subsection{Symmetry-induced network folding and orbit augmentation}
\label{sec:symmetry_folding_main}

When a calibrated task operator preserves a physical symmetry, symmetry related samples describe the same operator law under the corresponding representation transformation. This follows the usual role of symmetry in physical laws and is also consistent with geometric data augmentation in computer vision, where transformed images, including translated, scaled, or rotated images, are used as additional training samples~\cite{noether1918invariante,lecun1998gradient,chen2020group}. Let \(\Gamma\) be a compact symmetry group acting on the task-specific input and target fields through \(\rho_{\mathrm{in}}\) and \(\rho_{\mathrm{out}}\). If
\begin{equation}
\label{eq:symmetry_equivariance}
\mathcal{F}_{\star}
\left[
\rho_{\mathrm{in}}(\gamma)\bm{Z}_{\star}
\right]
=
\rho_{\mathrm{out}}(\gamma)
\mathcal{F}_{\star}
\left[
\bm{Z}_{\star}
\right],
\qquad
\forall \gamma\in\Gamma,
\end{equation}
then a valid calibrated pair \((\bm{Z}_{\star},\bm{Y}_{\star})\) gives the orbit-augmented pairs
\begin{equation}
\label{eq:symmetry_orbit_augmentation}
(\bm{Z}_{\star},\bm{Y}_{\star})
\mapsto
\left\{
\left(
\rho_{\mathrm{in}}(\gamma)\bm{Z}_{\star},
\rho_{\mathrm{out}}(\gamma)\bm{Y}_{\star}
\right)
:
\gamma\in\Gamma
\right\}.
\end{equation}

The same symmetry can also be used at the component level. For symmetry related components \(q\in\mathcal{Q}\), choose a canonical component \(q_0\) and a transformation \(\gamma_q\) mapping \(q\) to \(q_0\). With local window extraction \(\mathbf{z}_{\star,i}=\mathcal{E}_i[\bm{Z}_{\star}]\), the folded local sample is written as
\begin{equation}
\label{eq:symmetry_component_folding}
\widetilde{\mathbf{z}}_{\star,i,q}
=
P_{q_0}^{\mathrm{in}}
\mathcal{E}_i
\left[
\rho_{\mathrm{in}}(\gamma_q)\bm{Z}_{\star}
\right],
\qquad
\widehat{\widetilde{\mathbf{y}}}_{\star,i,q}
=
P_{q_0}^{\mathrm{out}}
\left(
\rho_{\mathrm{out}}(\gamma_q)\bm{Y}_{\star}
\right)_i .
\end{equation}
The corresponding canonical local reconstruction is
\begin{equation}
\label{eq:symmetry_folded_reconstruction}
\widehat{\widetilde{\mathbf{y}}}_{\star,i,q}
=
\sum_{j\in\mathcal{S}(i)}
\bm{G}_{\star,i,j}^{\theta}
\left(
\widetilde{\mathbf{z}}_{\star,i,q}
\right)
\widetilde{\bm{u}}_{i,j,q},
\qquad
q\in\mathcal{Q}.
\end{equation}
Thus, orbit augmentation provides additional valid input and target pairs, while componentwise folding reduces symmetry related components to a common local learning problem. This avoids separate parameterization of equivalent components, in line with the parameter sharing principle used in group equivariant neural networks~\cite{cohen2016group}. The reduction is valid only when the calibrated target respects the same symmetry. When zero point calibration or residualization is applied, the resulting target must still transform under \(\rho_{\mathrm{out}}\). We therefore do not use folding in problems with symmetry breaking sources, potentials, anisotropy, or direction dependent forcing, unless these factors are explicitly included in the input representation. The derivation is given in~\ref{sec:symmetry_folding}.

\section{Results}
\label{sec:results}

We next demonstrate and discuss the performance of LGNO through a rich set of numerical examples, covering physical problems ranging from one-dimensional to three-dimensional, and involving both source identification and dynamic prediction of physical fields.

\subsection{Statement of network architecture and data training}
The LGNO PDE settings, sampling configurations, network sizes, and corresponding
test relative \(L^2\) errors are reported in Table~\ref{tab:concise_results}. The network parameters were optimized using Adam, with an adaptive learning rate scheduler that reduces the step size once the training loss reaches a plateau. Gradient clipping and early stopping were used to improve training stability.
The complete results, including all tested models, hidden widths, parameter counts, training losses, test relative \(L^2\) errors, and error ratios, are provided in~\ref{Appendix:all_results}.

\begin{table}[htbp]
\centering
\caption{Best performance of LGNO across PDE benchmarks}
\begin{adjustbox}{max width=\linewidth}
\begin{tabular}{cccccc}
\hline
\textbf{Benchmark} & \textbf{Nonlinearity} & \textbf{Type} & \textbf{Parameters} & \textbf{Test rel. $L^2$} \\
\hline
1D Diffusion & & time & 23 & $1.2\%$ \\
2D Burgers & $\checkmark$ & space &  105 & $5.0\% $ \\
2D NS (two equations) & $\checkmark$ & time &  1188 & $0.8\%$ \\
2D SE (Gross-Pitaevskii) & $\checkmark$ & space &  370 & $11.6\%$ \\
3D SE (Perturbed harmonic) &  & time & 988 & $0.3\%$ \\
\hline
\end{tabular}
\end{adjustbox}
\label{tab:concise_results}
\end{table}

To ensure reproducibility, all experiments were conducted with fixed random seeds for data splitting, model initialization, and stochastic training procedures. The default random seed was set to 42 unless otherwise stated.
We test the proposed method on a set of PDE benchmarks covering linear and nonlinear, source-identification and temporal-evolution cases. 
The test data were strictly excluded from training in all cases. Experiments were run on a workstation equipped with an Intel i9-14900K 24-core CPU and an NVIDIA RTX 4080 SUPER GPU.

Moreover, we adopt a width controlled comparison, where the suffix ``-k'' denotes the hidden width of each model. This setting reduces the effect of parameter scaling and allows us to focus on the influence of architectural inductive bias.

\subsection{1D examples}

\paragraph{1D Diffusion}

We consider the one dimensional periodic diffusion equation
\begin{equation}
\frac{\partial u}{\partial t}
=
\kappa \frac{\partial^2 u}{\partial x^2},
\qquad
x\in[0,1),\quad t\in[0,1],
\end{equation}
where \(\kappa=0.1\). The equation is discretized on a uniform periodic grid with \(N_x=32\) spatial points and \(N_t=300\) temporal points. 
One trajectory is used for training and six independently generated trajectories are used for testing.

\begin{figure}[htb]
    \centering
    \includegraphics[width=\linewidth]{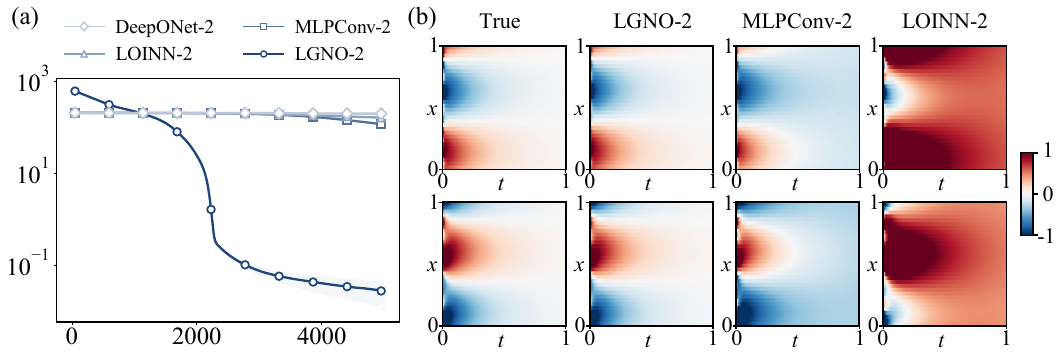}
    \caption{Training losses and rollout predictions for the one dimensional diffusion problem.
    (a) Training loss curves of LGNO-2, LOINN-2, MLPConv-2, and DeepONet-2.
    (b) Reference solutions and rollout predictions for two test cases. From left to right, the columns show the reference solution and the predictions of LGNO-2, MLPConv-2, and LOINN-2. The horizontal axis denotes time and the vertical axis denotes the spatial coordinate.}
    \label{fig:1ddiffusion}
\end{figure}

We compare LGNO-2 with LOINN-2, MLPConv-2, and DeepONet-2. All models are trained on the same single trajectory. As shown in \figref{fig:1ddiffusion}(a), the loss of LGNO-2 decreases rapidly after the initial stage and reaches a much lower value than those of the three baseline models. The diffusion operator has a compact three point stencil, which agrees well with the local structure used by LGNO.

\figref{fig:1ddiffusion}(b) shows the rollout results for two test cases. LGNO-2 reproduces the spatial patterns and their decay throughout the time interval. MLPConv-2 captures the general diffusion trend, but its amplitude and long time behavior differ from the reference solution. The prediction of LOINN-2 shows a larger deviation. DeepONet-2 is omitted from the rollout panels because its prediction error is much larger.

\begin{table}[htb]
    \centering
    \caption{Rollout errors on the six test trajectories.}
    \label{tab:1d_diffusion_errors}
    \begin{adjustbox}{max width=\linewidth}
    \begin{tabular}{lccccc}
        \toprule
        Model & Parameters & Min. error & Max. error
        & Overall error & Relative error \\
        \midrule
        LGNO-2     & 23 & 0.009 & 0.014  & 0.012  & \(1.0\times\) \\
        MLPConv-2  & 29 & 0.091 & 1.611  & 0.947  & \(78.3\times\) \\
        LOINN-2  & 35 & 0.662 & 4.878  & 3.068  & \(253.7\times\) \\
        DeepONet-2 & 95 & 5.284 & 42.573 & 23.492 & \(1942.6\times\) \\
        \bottomrule
    \end{tabular}
    \end{adjustbox}
\end{table} 

Table~\ref{tab:1d_diffusion_errors} reports the rollout errors on the six test trajectories. LGNO-2 gives the lowest error in every case. Its overall relative \(L_2\) error is \(1.21\times10^{-2}\), compared with \(9.47\times10^{-1}\) for MLPConv-2, the most accurate baseline. This corresponds to an error reduction of approximately \(98.7\%\). The results suggest that the local zero consistent structure allows LGNO to recover the discrete diffusion operator from a single training trajectory.

\subsection{2D examples}

\paragraph{2D Navier--Stokes}
\label{Para:2DNS}
We test LGNO on a 2D projected incompressible Navier--Stokes
benchmark in velocity form. After absorbing the pressure gradient into the
effective forcing, the equations used for training are
\begin{equation}
\begin{cases}
u_t + u\,\partial_x u + v\,\partial_y u
    = \nu \Delta u + \widetilde{g}_u, \\
v_t + u\,\partial_x v + v\,\partial_y v
    = \nu \Delta v + \widetilde{g}_v,
\end{cases}
\label{eq:ns_velocity}
\end{equation}
where $(u,v)$ are the two velocity components, $\nu=0.002$ is the viscosity,
and $(\widetilde{g}_u,\widetilde{g}_v)$ denotes the pressure absorbed forcing.
For the present data set, the external forcing is zero, so this term represents
the negative pressure gradient. The divergence-free structure is inherited
from the data-generation procedure.

The periodic domain $[-1,1)\times[-1,1)$ is discretized on a $32\times32$
grid. The trajectory contains 1000 frames with time step
$\Delta t=0.002$, covering $t\in[0,2.00)$. Only the first 20 frames,
$t\in[0,0.04)$, are used for training. Starting from the initial velocity
field, each model is then rolled out autoregressively over the complete
1000-frame trajectory.

\begin{figure}[htb]
    \centering
    \includegraphics[width=\linewidth]{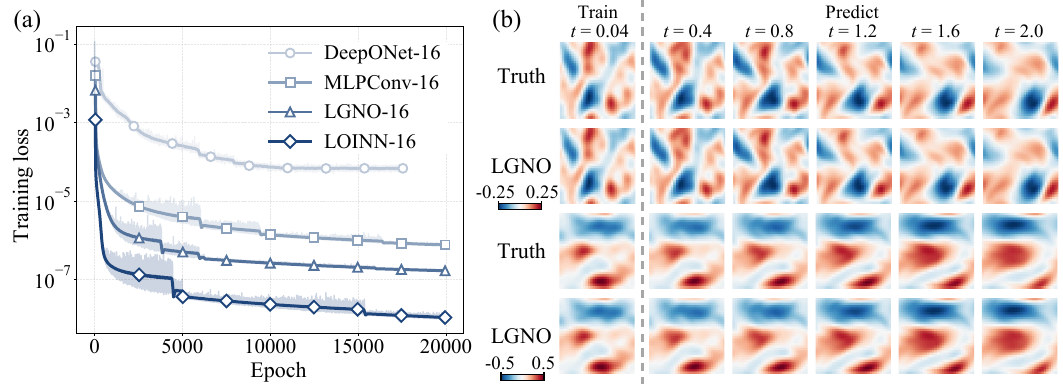}
    \caption{Learning and long time rollout of two-dimensional incompressible
    Navier--Stokes flow.
    (a) Training-loss curves of the compared models.
    (b) Reference and LGNO-16-predicted velocity fields. The first two rows
    show the horizontal velocity \(u\), and the last two rows show the vertical
    velocity \(v\). For each component, the upper row gives the reference
    solution and the lower row gives the LGNO-16 prediction. The model is
    trained using the first 20 frames, \(t\in[0,0.04)\), and rolled out to the
    final recorded time \(t=2.0\). Snapshots are shown
    at \(t=0.04,0.4,0.8,1.2,1.6,\) and \(2.0\).}
    \label{fig:2DNS}
\end{figure}

 \figref{fig:2DNS} summarizes the training behavior and the long time rollout result. As shown in \figref{fig:2DNS}(a), LGNO-16 reaches a low training error with a relatively small number of parameters. The rollout comparison in \figref{fig:2DNS}(b) further shows that, although the model is trained only about \(2\%\) of the full trajectory, it keeps the main vortex structures and remains stable until \(t=2.0\). Rollout results of the baseline models are given in~\ref{Appendix:2DNS}.

These results also show that a low training error, or even a low one step time-derivative error, is not sufficient for stable long time prediction. For example, LOINN-16 obtains the lowest training MSE, \(3.06\times10^{-8}\), and the lowest one-step time-derivative relative \(L^2\) error, \(2.50\times10^{-3}\), but its autoregressive rollout diverges. MLP-16 and FNO-16 also reach small training losses, while their rollout relative \(L^2\) errors increase to \(1.34\) and \(5.14\), respectively.

Table~\ref{tab:2DNS_model_comparison} summarizes the results for all tested configurations. The step error is computed from the saved one step time derivative test set of each implementation, and the rollout error is computed over both velocity components, all grid points, and all 1000 frames. LGNO-16 gives the best finite rollout relative \(L^2\) error, \(8.14\times10^{-3}\), with only \(1{,}188\) trainable parameters. Increasing the LGNO width from 16 to 32 does not further improve the rollout accuracy too much in this case, while LGNO-8 still gives a competitive result with only 548 parameters. Among the baseline models, MLPConv-16 gives the lowest rollout relative \(L^2\) error, \(1.74\times10^{-2}\).

\begin{table}[htb]
\centering
\caption{Accuracy and model size results for the 2D Navier--Stokes
benchmark. ``Train MSE'' is the minimum training loss. ``Step relative errors'' measure the predicted time derivative of each step, and ``rollout relative errors'' measure the complete
autoregressive velocity trajectory.}
\label{tab:2DNS_model_comparison}
\vspace{0.5em}
\begin{adjustbox}{max width=\linewidth}
\begin{tabular}{lrrrr}
\hline
Method
& Params
& Train MSE
& Step rel. $L^2$
& Rollout rel. $L^2$ \\
\hline
LGNO-16
& $1{,}188$
& $3.32\times10^{-7}$
& $1.38\times10^{-2}$
& $\mathbf{8.14\times10^{-3}}$ \\

MLPConv-16
& $610$
& $2.82\times10^{-6}$
& $2.76\times10^{-2}$
& $1.74\times10^{-2}$ \\

DeepONet-16
& $34{,}738$
& $4.49\times10^{-3}$
& $1.44$
& $1.01$ \\

LOINN-16
& $1{,}170$
& $3.06\times10^{-8}$
& $2.50\times10^{-3}$
& \textit{overflow} \\

\hline
\end{tabular}
\end{adjustbox}
\end{table}

\paragraph{2D Burgers}
\label{2D Burgers}

We further test LGNO on a scalar two-dimensional viscous Burgers-type
advection--diffusion benchmark. In this setting, the scalar field $u$ is
advected by the velocity $(u,u)$, and the nonlinear operator is written as
\begin{equation}
    f = (u,u)\cdot\nabla u - \nu \Delta u
      = u\,\partial_x u + u\,\partial_y u - \nu \Delta u,
\end{equation}
where the viscosity is set to $\nu=0.01$. This benchmark is used as a nonlinear
scalar operator-learning test case, rather than as the standard vector-valued
two-component Burgers system. The data are generated on a $512\times512$ grid
over the periodic domain $[-1,1)^2$.

To test data efficiency, we use only $5$ samples in total, with $1$ sample for
training and $4$ samples for testing. This extremely low-data setting makes
generalization difficult and exposes the tendency of over-parameterized models
to overfit. In the following comparison, we primarily consider the low-capacity
setting with hidden size $8$.

\begin{figure}[htb]
	\centering
\includegraphics[width=\linewidth]{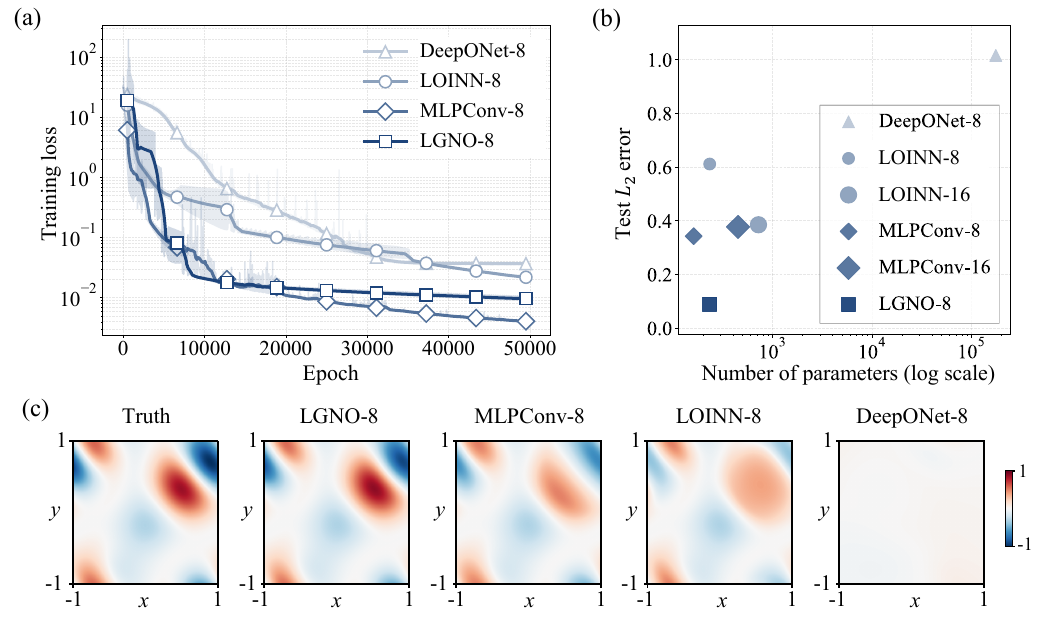}
	\caption{
Performance comparison on the 2D Burgers equation with hidden
size $8$.
(a) Training loss curves of the compared models.
(b) Pareto plot showing the relation between model size and test error.
(c) Qualitative comparison of operator predictions on a representative test
sample, including the reference solution and four models with hidden size $8$.
Additional results are shown in appendix Fig.~\ref{[Appendix]2DBurger}.}
\label{fig:burgers_results}
\end{figure}

 \figref{fig:burgers_results} compares the performance of different models on the 2D Burgers benchmark focusing on hidden size $8$, while Table~\ref{tab:concise_results} reports the best LGNO configuration at width 4. \figref{fig:burgers_results}(a) shows the training loss curves of the compared
models. MLPConv achieves a minimum training loss of
$2.65\times10^{-3}$, which is lower than the
$5.88\times10^{-3}$ obtained by LGNO. However, the Pareto plot in \figref{fig:burgers_results}(b) shows that the lower training loss does not
translate into better test accuracy. LGNO uses only $233$ trainable parameters
and reaches a test relative $L_2$ error of $9.29\times10^{-2}$. In comparison,
MLPConv, the best-performing baseline in this setting, uses $161$ parameters
but has a substantially larger test error of $3.54\times10^{-1}$. Thus, LGNO
reduces the test error by $73.77\%$, corresponding to a $3.81$-fold reduction
relative to the best baseline. This result indicates that LGNO achieves a
better balance between model size and test accuracy.

The qualitative comparison in \figref{fig:burgers_results}(c) further
supports this observation. On the representative test sample, the LGNO
prediction shows the closest visual agreement with the reference solution,
whereas the baseline models exhibit more noticeable deviations. These results
suggest that, although MLPConv can fit the single training sample more closely,
LGNO benefits from its local operator parameterization. Instead of directly
fitting an input-output map, LGNO learns adaptive local weights in a stencil
form, which helps reduce overfitting and improves test accuracy in the
extremely low data regime. Additional prediction results under other operating
conditions are provided in~\ref{Appendix:2D Burgers}.

\paragraph{2D Gross--Pitaevskii}
\label{paragraph:2DGP}

We next apply LGNO to the two dimensional Gross--Pitaevskii equation (GPE),
which is commonly used to describe Bose--Einstein condensates:
\begin{equation}
    f = \hat{H}\psi
    = -\frac{\hbar^2}{2m}\nabla^2\psi
    + V(\mathbf{r})\psi
    + g|\psi|^2\psi .
\end{equation}

The data are generated on a $64\times64$ uniform grid over a rectangular
periodic domain. Periodic boundary conditions are imposed on both the complex
wavefunction $\psi$ and the potential field $V$. The Hamiltonian action
$\hat{H}\psi$ is computed using a periodic finite difference discretization,
with the Laplacian evaluated by periodic stencil operations.

The Gross--Pitaevskii operator acts on a complex wavefunction \(\psi\). We first write
\begin{equation}
    \psi=\psi_{\mathrm R}+i\psi_{\mathrm I},
    \qquad
    \psi_{\mathrm R}=\mathrm{Re}(\psi),
    \qquad
    \psi_{\mathrm I}=\mathrm{Im}(\psi),
\end{equation}
and define the amplitude-squared field
\begin{equation}
    |\psi|^2=\psi_{\mathrm R}^2+\psi_{\mathrm I}^2 .
\end{equation}
The target is the Hamiltonian transformed field
\begin{equation}
    \hat H\psi
    =
    \mathrm{Re}(\hat H\psi)
    +
    i\,\mathrm{Im}(\hat H\psi).
\end{equation}

Since the same real valued Hamiltonian operator acts on both the real and
imaginary components, we do not treat \(\mathrm{Re}(\hat H\psi)\) and
\(\mathrm{Im}(\hat H\psi)\) as two unrelated output channels. Instead, the
coefficient-generation network is shared between the two component wise
branches. For the real component, the model uses the local stencil of
\(\psi_{\mathrm R}\), together with the local conditioning information provided
by \(V\) and \(|\psi|^2\), to predict
\begin{equation}
    f_{\mathrm R}
    =
    \mathrm{Re}(\hat H\psi).
\end{equation}
The imaginary component is treated in the same way. 

Equivalently, the component wise local construction can be written as
\begin{equation}
    [\mathrm{stencil}(\psi_{\mathrm R}),|\psi|^2,V]
    \longmapsto
    \mathrm{Re}(\hat H\psi),
    \qquad
    [\mathrm{stencil}(\psi_{\mathrm I}),|\psi|^2,V]
    \longmapsto
    \mathrm{Im}(\hat H\psi).
\end{equation}

Thus, the real and imaginary branches share the
same coefficient generation mechanism, consistent with the symmetry induced
folding principle discussed in Section~\ref{sec:symmetry_folding_main}.

This folding strategy improves sample efficiency in two ways. First, each local
prediction is reduced from a coupled complex output problem to a component-wise
scalar output problem, so the effective learning complexity is smaller than
that of the original full channel formulation. Second, every complex valued
training sample provides two structurally equivalent component wise learning
problems, one for the real branch and one for the imaginary branch. Therefore,
the model can exploit the shared Hamiltonian structure between the two
components without introducing separate coefficient generators. This design is
especially useful for extremely small training sets, where directly learning
the full complex operator may be underdetermined.

\begin{figure}[htb]
    \centering
    \includegraphics[width=\linewidth]{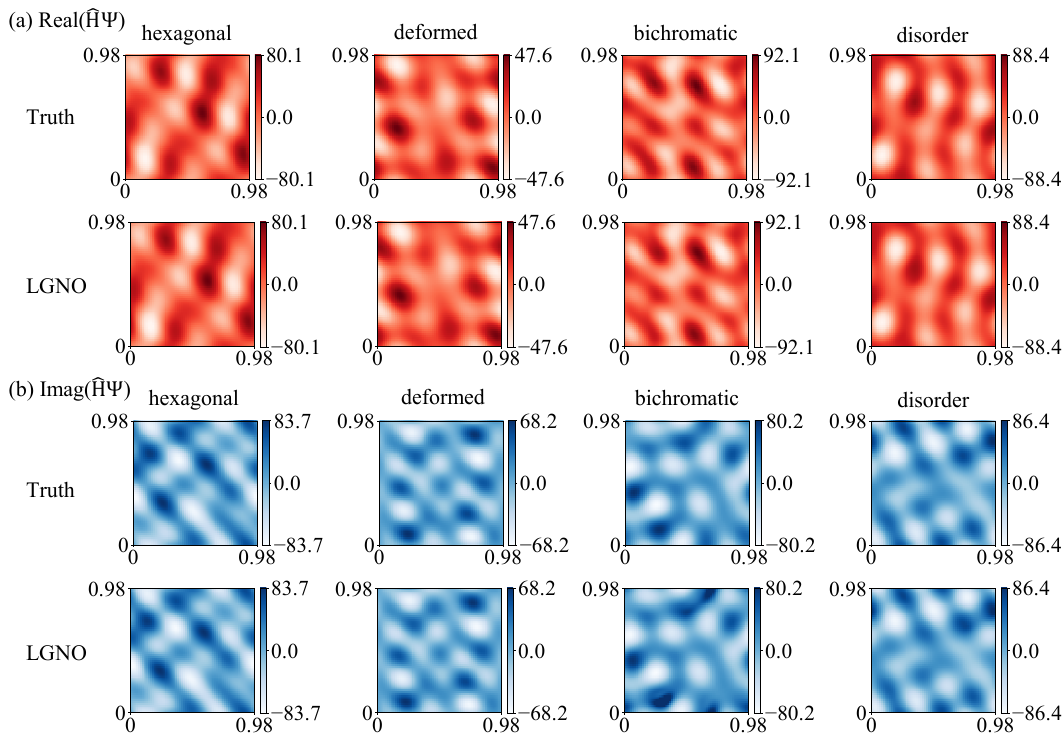}
    \caption{Reference and predicted Gross--Pitaevskii operator responses on unseen periodic potentials, including hexagonal, deformed, bichromatic, and disorder cases. (a) The real part of the source field. (b) The imaginary part of the source field.}
    \label{fig:2dgpe}
\end{figure}

The training set uses optical lattice potentials, whereas the test set includes
four unseen potential types: hexagonal, deformed, bichromatic, and disorder.
Details of the dataset construction and test potentials are given in~\ref{Appendix:2DGP}. This setting tests whether the model can
generalize from one potential family to other periodic potentials not seen
during training. The results presented in \figref{fig:2dgpe} show that LGNO is
able to reconstruct the field source distributions under several unseen
periodic potentials, although the bichromatic case remains more challenging
than the other test cases.

The optimized LGNO model uses a 5-point local stencil with hidden width \(12\).
Different from the previous explicit prior formulation, the \(V\psi\) term is
not treated as a known physical contribution. Instead, \(V\) is used only as a
local conditioning input to the coefficient generation MLP, and the full
normalized stencil weights are learned from data. The model contains only
\(370\) trainable parameters. The final overall test relative \(L^2\) error,
computed by merging all four unseen test cases and both output components, is
\(11.61\%\). The relative \(L^2\) errors for the real and imaginary components
of each test potential are reported in Table~\ref{tab:gpe_errors}.

\begin{table}[htb]
    \centering
    \caption{Relative \(L^2\) errors (\%) of LGNO on unseen 2D Gross--Pitaevskii test potentials. The average is computed as the arithmetic mean over the four test potentials.}
    \label{tab:gpe_errors}
    \begin{adjustbox}{max width=\linewidth}
    \begin{tabular}{c|cccc|c}
        \toprule
        Component & Hexagonal & Deformed & Bichromatic & Disorder & Average \\
        \midrule
        \(\mathrm{Re}(\hat{H}\psi)\) & 4.51 & 1.78 & 10.19 & 1.03 & 4.38 \\
        \midrule
        \(\mathrm{Im}(\hat{H}\psi)\) & 4.40 & 1.93 & 21.57 & 1.83 & 7.43 \\
        \bottomrule
    \end{tabular}
    \end{adjustbox}
\end{table}

Among the four unseen test potentials, the disorder case gives the smallest
real-part error, \(1.03\%\), while the deformed case gives the smallest
imaginary-part error, \(1.93\%\). The deformed and disorder potentials are both
predicted accurately, with errors below \(2\%\) for both components. The
hexagonal case also remains stable, with errors of \(4.51\%\) and \(4.40\%\) for
the real and imaginary parts, respectively. The bichromatic case is the most
difficult extrapolation case, with errors of \(10.19\%\) and \(21.57\%\) for the
real and imaginary parts, respectively.

For the hexagonal, deformed, and disorder potentials, the relative errors are
mostly below \(5\%\), showing that the learned local stencil operator can
generalize reasonably well from optical-lattice training potentials to several
unseen periodic potential families. The larger error on the bichromatic
potential indicates that this case introduces stronger out of distribution
features, especially for the imaginary component. Therefore, the results suggest
that the LGNO successfully captures the dominant local
Hamiltonian structure, but its extrapolation accuracy depends on how closely
the unseen potential family matches the local patterns observed during
training.

\subsection{3D examples}

\paragraph{3D time dependent Schr\"odinger equation}

We next test LGNO on the three dimensional time dependent Schr\"odinger
equation,
\begin{equation}
i\hbar \frac{\partial \psi}{\partial t}
=
\hat{H}\psi
=
-\frac{\hbar^2}{2m}\nabla^2 \psi
+
V(x,y,z)\,\psi,
\label{eq:Schrodinger3D}
\end{equation}
where $\psi(x,y,z,t)$ is the complex valued wavefunction and $\hat{H}$ denotes
the Hamiltonian operator. Here, we use dimensionless units with $\hbar=1$ and $m=1$,
so the kinetic term becomes $-\frac{1}{2}\nabla^2\psi$.

The data are generated on a $16\times16\times16$ Cartesian grid over
$[-2,2]^3$. Spatial derivatives are computed using second order central finite
differences, and the wavefunction is advanced in time with an explicit Euler
scheme. We generate one trajectory with $N_t=5000$ time steps and
$\Delta t=10^{-4}$. The first $20\%$ of the trajectory is used for training,
and the full time horizon is used to evaluate rollout behavior.

The initial condition is a localized Gaussian wave packet with a plane wave
phase in the $x$ direction,
\begin{equation}
\psi(x,y,z,0)
=
\exp\left(
-\frac{x^2+y^2+z^2}{2\sigma^2}
\right)
e^{i k_0 x},
\qquad
\sigma = 0.8,\quad k_0 = 2.0.
\label{eq:initial}
\end{equation}
This gives a wave packet with nonzero momentum along the $x$ direction.

The external potential is defined as an anisotropic harmonic trap with a smooth
sinusoidal perturbation,

\begin{figure}[!htbp]
	\centering
	\includegraphics[width=\linewidth]{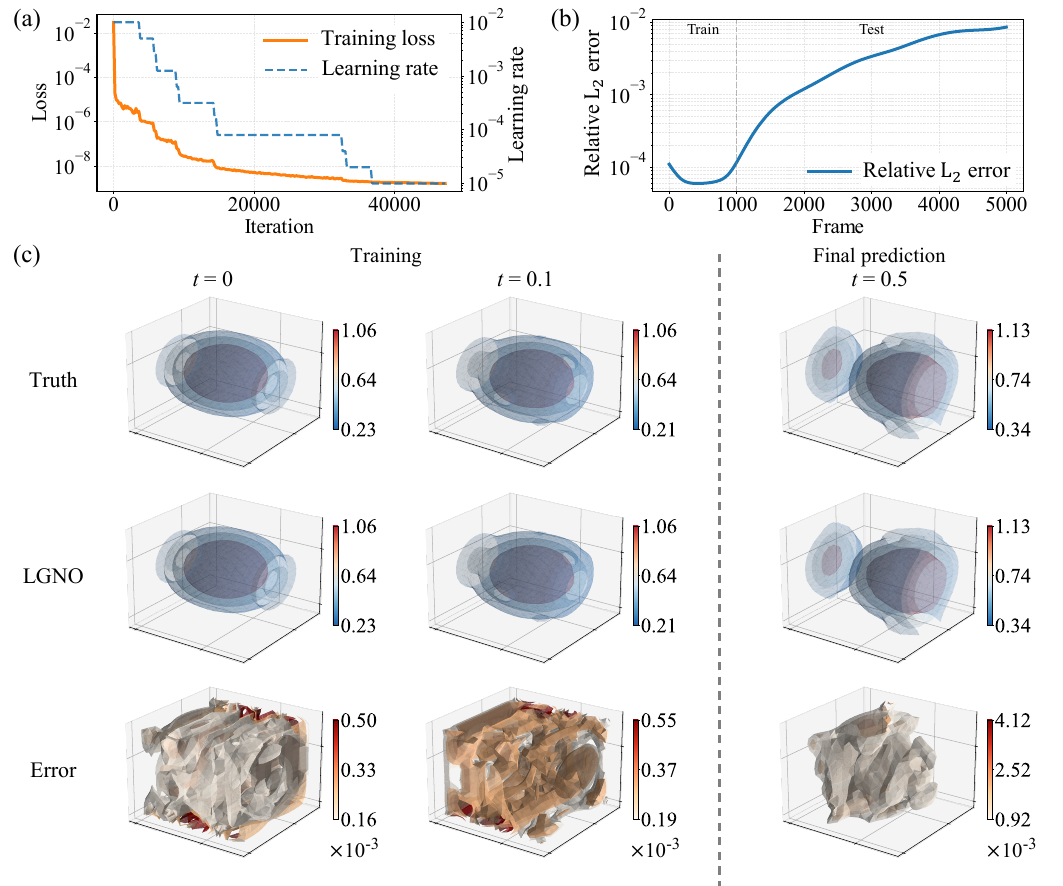}
	\caption{3D operator learning for the Schr\"odinger system.
    (a) Training curves. The solid line shows the training loss, and the dashed
line shows the learning rate schedule.
(b) Rollout error over time. The model is trained on early time data and then
evaluated on later frames.
(c) Isosurface comparison of the operator-response amplitude $|H\psi|$.
From top to bottom, the rows show the reference solution, the LGNO prediction,
and the pointwise relative error. The columns correspond to $t=0$, $t=0.1$,
and $t=0.5$. The first two snapshots are within the training interval, while
the last one is the final predicted frame.}
\label{fig:3DShro_3D}
\end{figure}

\begin{equation}
V(x,y,z)
=
\frac{1}{2}
\left(
\omega_x^2 x^2 + \omega_y^2 y^2 + \omega_z^2 z^2
\right)
+
\alpha \sin(2x)\cos(3y)\sin(z),
\label{eq:potential}
\end{equation}

where $(\omega_x,\omega_y,\omega_z)=(1.0,\,1.5,\,0.8)$ and $\alpha=0.3$.
This potential introduces both anisotropic confinement and spatial oscillations,
leading to changes in the wave packet shape and phase over time.

The local operator is represented by a three layer MLP with SiLU activations and hidden dimension $16$, applied pointwise on the grid. After training, the optimized LGNO model is adopted to predict the Hamiltonian operator response $\hat{H}\psi$ at different time instants. \figref{fig:3DShro_3D} summarizes the three dimensional learning results. As shown in \figref{fig:3DShro_3D}(a), the curves of training loss and the learning rate schedule demonstrated a stable and convergent optimization process. \figref{fig:3DShro_3D}(b) depicts the temporal variation of rollout error. The model is trained on early time data and evaluated on later frames, with the largest relative $L^2$ error of $8.58\times10^{-3}$ appearing at the final frame $t=0.5$, while the total average relative $L^2$ error is only 0.3\%. \figref{fig:3DShro_3D}(c) provides a visual check of the operator prediction at representative stages of the rollout. The predicted responses remain close to the reference solution with the reference field in both spatial morphology and amplitude distribution.

\begin{figure}[!htbp]
	\centering
	\includegraphics[width=\linewidth]{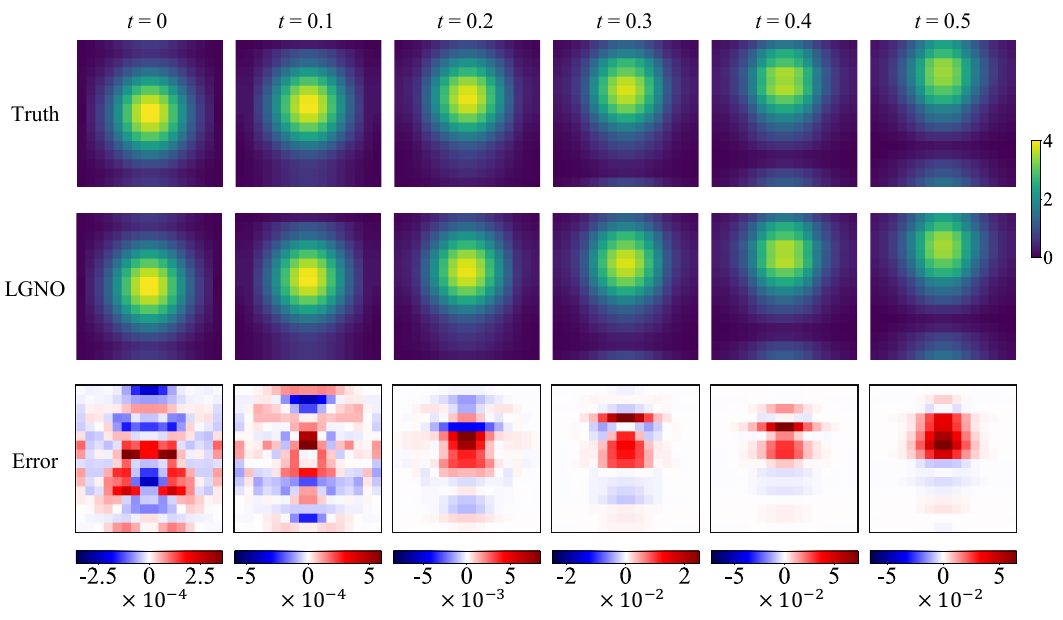}
    \caption{Hamiltonian operator prediction for the three-dimensional Schr\"odinger
equation with a symmetry-breaking potential. The central $xy$-slice at $z=0$
is shown. From top to bottom, the rows give the reference solution, the LGNO
prediction, and the absolute error. The reference and prediction use the same
color scale for the amplitude $|\hat{H}\psi|$, with values in $[0,4]$. The
error is plotted with a separate color scale from $2.5\times10^{-4}$ to
$5\times10^{-2}$.}
\label{fig:3DShro}
\end{figure}

\figref{fig:3DShro} further shows the central $xy$-slice at $z=0$ for the predicted field source. The reference and prediction share the same amplitude color scale, while the error is plotted separately. Slice-wise comparison confirms that the LGNO model reconstructs the primary three-dimensional operator response with high accuracy, achieving a maximum relative error of less than 5\%.

\subsection{Robustness}
\label{sec:robustness}

We further evaluate the robustness of LGNO on the two dimensional Burgers
benchmark by perturbing the input field with multiplicative noise. For a clean
input field \(u\), the noisy field is defined as
\begin{equation}
    u_{\mathrm{noisy}} = u(1+\alpha \eta),
\end{equation}
where \(\eta\) is generated from a uniform random field and then spatially
smoothed. The parameter \(\alpha\) controls the perturbation strength. In this
experiment we consider
\begin{equation}
    \alpha \in \{0.00,0.01,\ldots,0.10\}.
\end{equation}

For each noise level, LGNO-8 and MLPConv-8 are trained under the same setting:
hidden dimension \(8\), learning rate \(10^{-2}\), \(10000\) maximum epochs, and patience \(1000\). MLPConv-8 is used as the local baseline because it has the closest stencil based structure among the compared methods.

\begin{figure}[htb]
    \centering
    \includegraphics[width=0.96\linewidth]{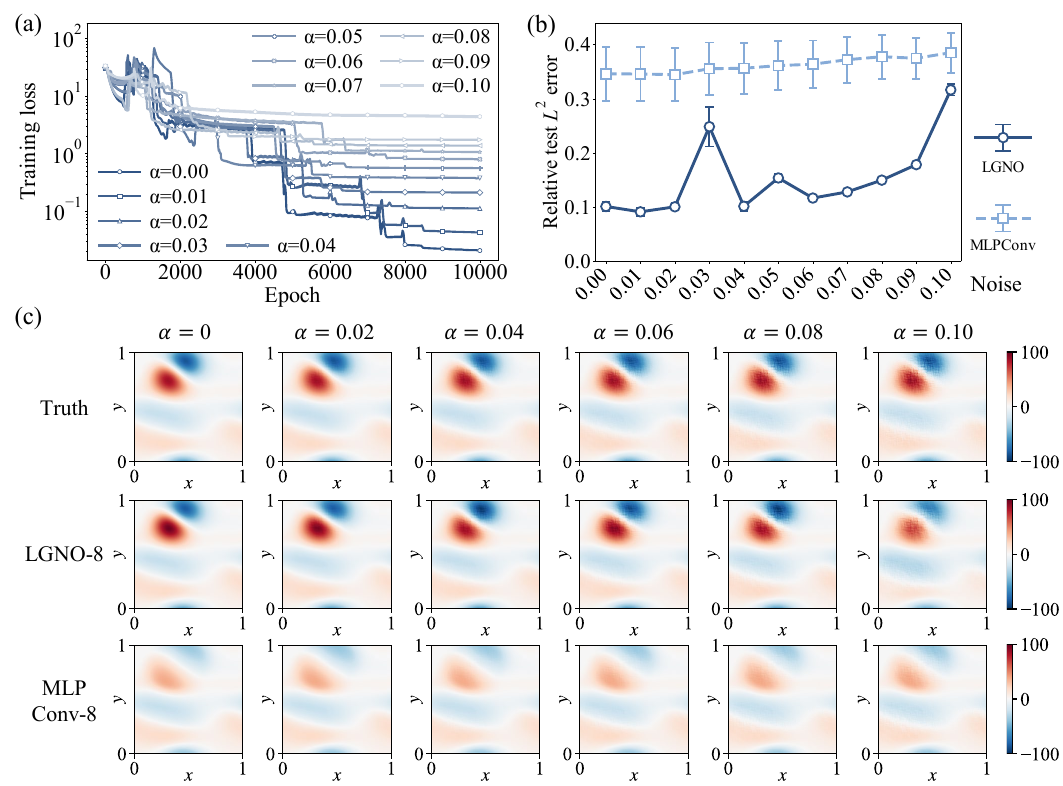}
    \caption{
    Robustness evaluation under multiplicative input noise.
    (a) LGNO training loss curves for noise levels
    \(\alpha=0.00,\ldots,0.10\).
    (b) Mean relative \(L^2\) test error of LGNO-8 and MLPConv-8 over five
    evaluation seeds \(447\)--\(451\). Error bars visualize \(0.3\sigma\),
    where \(\sigma\) is the standard deviation across the five seeds.
    (c) Representative qualitative comparison at normalized noise ratios
    \(0,0.2,0.4,0.6,0.8,1.0\), corresponding to
    \(\alpha=0.00,0.02,0.04,0.06,0.08,0.10\). Each group shows the reference
    solution, the LGNO-8 prediction, and the baseline prediction.
    }
    \label{fig:noise}
\end{figure}

 \figref{fig:noise}(a) shows that LGNO can be optimized consistently across
all tested noise levels. Although the loss trajectories become less smooth as
the perturbation increases, the model remains trainable throughout the whole
range \(\alpha\in[0,0.10]\). \figref{fig:noise}(b) gives the quantitative
comparison with MLPConv-8. The reported errors are averaged over five
independent evaluation seeds \(447\)--\(451\), and both models are tested on
exactly the same noisy level inputs. LGNO-8 achieves a lower mean relative \(L^2\)
error at every tested noise level. The advantage is large for weak and
moderate perturbations, and it remains positive even at the strongest tested
noise level \(\alpha=0.10\), where LGNO-8 obtains an average error of
\(3.166\times10^{-1}\) compared with \(3.853\times10^{-1}\) for MLPConv-8.

The qualitative examples in \figref{fig:noise}(c) further support this
observation. As the normalized noise ratio increases from \(0\) to \(1.0\), the
input field becomes progressively more perturbed. LGNO-8 still preserves the
dominant spatial structure of the reference solution and avoids the larger
local deviations observed in the baseline prediction. These results indicate
that the learned local geometric operator remains effective when the input is
corrupted by moderate multiplicative noise.

Overall, the robustness experiment shows that LGNO-8 is stable under the
tested perturbation range. Compared with the MLPConv-8 baseline, it produces
lower average relative \(L^2\) error for all noise levels
\(\alpha=0.00,\ldots,0.10\), with relative improvements ranging from
\(17.8\%\) at \(\alpha=0.10\) to \(73.6\%\) at \(\alpha=0.01\).

\section{Conclusion}
\label{sec:conclusion}

This work presented LGNO, a lightweight and interpretable local gradient neural operator for PDE source identification and temporal field prediction. By using gradient-aware stencil reconstruction, LGNO separates stencil coefficient learning from field reconstruction and embeds the locality of differential operators into the network design. Together with symmetry-induced folding and orbit augmentation, this structure improves sample efficiency and reduces redundant parameters under limited data. Numerical tests cover linear and nonlinear, static and dynamic, and low- and high-dimensional benchmarks against representative neural operator baselines.

Under the unified low-data comparison protocol, LGNO reduces the test error by about \(74.5\%\) on average. The reductions are especially clear in the 2D Burgers and 3D Schr\"odinger cases, where the errors decrease by \(85.9\%\) and \(84.2\%\), respectively. In the Navier--Stokes rollout test, LGNO reduces the error by \(53.2\%\) compared with the best stable baseline. These results are obtained with compact models, using only \(23\) parameters for 1D diffusion, \(105\) for 2D Burgers, \(370\) for 2D Gross--Pitaevskii, and fewer than \(1200\) for the Navier--Stokes and Schr\"odinger examples.

The core modules of LGNO provide distinct practical benefits. Symmetry folding reuses equivalent field components to increase effective samples while controlling model size. Zero-consistent reconstruction removes static offsets and avoids artificial zero-state drift, while the error bound separates discretization and fitting errors. The Burgers noise tests further show that LGNO remains stable under moderate input perturbations. Future work will extend LGNO toward weakly nonlocal and multi-scale systems by incorporating lightweight global-context perception, while carefully balancing computational cost, parameter efficiency, and interpretability.

\section*{Acknowledgement}
The authors acknowledge the support of the National Natural Science Foundation of China (Grants No. 52505102, 12504544, 12525201, and 12432010). Baiming Zhang is supported by the Future Academic Star Program for Undergraduates at Zhejiang University.

\section*{Data availability}
All the datasets in the study are generated directly from the code.

\section*{Declaration of competing interest}
The authors declare that they have no known competing financial interests or personal relationships that could have appeared to influence the work reported in this paper.

\appendix
\renewcommand{\thetheorem}{\Alph{section}.\arabic{theorem}}

\section{Local stencil representation of discretized PDE operators}
\label{sec:discrete_kernel}

A finite-order local differential operator is evaluated from local field information. After compact-stencil discretization, this locality becomes a finite interaction pattern among neighboring grid values. For nonlinear or quasilinear operators, the local coupling coefficients may depend on the local state.

\begin{theorem}
\label{theorem:sparse}
Let \(\mathcal{L}\) be a finite-order local differential operator, and let \(\mathcal{L}_h\) be a compact-stencil discretization of order \(m\) on a uniform grid
\begin{equation}
\label{eq:uniform_grid_definition_pde}
\mathcal{G}_h
=
\{\bm{x}_i\}_{i=1}^{N}.
\end{equation}
For each interior node \(\bm{x}_i\), assume that the stencil \(\mathcal{S}(i)\) is finite and uniformly bounded independently of \(h\). Define the local input window by
\begin{equation}
\label{eq:local_window_pde}
\mathbf{z}_{\mathcal{L},i}
=
\left\{
(\bm{x}_j,\bm{u}_j)
:
j\in\mathcal{S}(i)
\right\}.
\end{equation}
For \(j\in\mathcal{S}(i)\), write \(\bm{u}_{i,j}:=\bm{u}_j\). Then the discrete response at \(\bm{x}_i\) depends only on \(\mathbf{z}_{\mathcal{L},i}\):
\begin{equation}
\label{eq:local_rule_pde}
(\mathcal{L}_h[\bm{u}])_i
=
\Psi_h
\left(
\mathbf{z}_{\mathcal{L},i}
\right),
\end{equation}
for some local discrete rule \(\Psi_h\). For smooth admissible fields, the consistency relation is
\begin{equation}
\label{eq:local_pde_discretization}
(R_h\mathcal{L}[\bm{u}])_i
=
(\mathcal{L}_h[\bm{u}])_i
+
\mathcal{O}(h^m),
\end{equation}
where \(R_h\) denotes restriction to the grid nodes. If the discrete operator is written in a state-dependent affine stencil form with respect to the contracted primary-field values, then
\begin{equation}
\label{eq:local_kernel_representation_pde}
(\mathcal{L}_h[\bm{u}])_i
=
\sum_{j\in\mathcal{S}(i)}
K_{i,j}
\left(
\mathbf{z}_{\mathcal{L},i}
\right)
\bm{u}_{i,j}
+
\bm{b}_i,
\end{equation}
where \(K_{i,j}(\mathbf{z}_{\mathcal{L},i})\in\mathbb{R}^{c_{\mathcal{L}}\times c_u}\) is the local coupling block from stencil node \(j\) to the response at node \(i\). Here, \(c_u\) is the number of primary-field components and \(c_{\mathcal{L}}\) is the number of response components of \(\mathcal{L}_h[\bm{u}]\). The additive term \(\bm{b}_i\in\mathbb{R}^{c_{\mathcal{L}}}\) is independent of the contracted primary-field values, although it may depend on the grid node or fixed prescribed quantities. If the operator, grid, and discretization are translation equivariant away from boundaries, the same interior local rule may be reused across grid points.
\end{theorem}

\begin{proof}
By strict locality, \(\mathcal{L}[\bm{u}](\bm{x}_i)\) is determined by the finite jet of \(\bm{u}\) at \(\bm{x}_i\). A compact-stencil discretization approximates this local jet by finite combinations of nodal values in \(\mathcal{S}(i)\). Hence the discrete response depends only on \(\mathbf{z}_{\mathcal{L},i}\), which gives Eq.~\eqref{eq:local_rule_pde}. The consistency estimate in Eq.~\eqref{eq:local_pde_discretization} follows from the assumed discretization order \(m\). If the local discrete dependence is represented in affine stencil form, the corresponding coefficient blocks can be collected as \(K_{i,j}(\mathbf{z}_{\mathcal{L},i})\), giving Eq.~\eqref{eq:local_kernel_representation_pde}. Under translation-equivariant interior settings, this local rule can be shared without assigning independent rules to different interior nodes.
\end{proof}

Vectorizing Eq.~\eqref{eq:local_kernel_representation_pde} over all valid grid nodes gives
\begin{equation}
\label{eq:sparse_matrix_form_pde}
\mathcal{L}_h[\bm{u}]
=
A_h(\bm{u}_h)\bm{u}_h
+
\bm{b}_h,
\end{equation}
where
\begin{equation}
\label{eq:restricted_primary_field_pde}
\bm{u}_h
=
R_h\bm{u}.
\end{equation}
For a fixed grid, the coordinate dependence in the local rules is included in the discrete operator \(A_h\). The matrix \(A_h(\bm{u}_h)\) is block sparse and satisfies
\begin{equation}
\label{eq:block_sparsity_pde}
A_{i,j}(\bm{u}_h)
=
\bm{0}
\qquad
\text{whenever}
\qquad
j\notin\mathcal{S}(i).
\end{equation}
Combining Eqs.~\eqref{eq:local_pde_discretization} and \eqref{eq:sparse_matrix_form_pde} yields
\begin{equation}
\label{eq:sparse_matrix_consistency_pde}
R_h\mathcal{L}[\bm{u}]
=
A_h(\bm{u}_h)\bm{u}_h
+
\bm{b}_h
+
\mathcal{O}(h^m).
\end{equation}

\section{Proofs of zero-consistent reconstruction and error estimates}
\label{sec:zero_consistency}

\paragraph{Zero-point calibration and bias-free reconstruction}
For a local window centered at \(\bm{x}_i\), let \(\mathbf{z}_{\star,i}^{0}\) denote the zero-state window in which the contracted primary-field values are set to zero, while coordinates and other prescribed conditioning channels, if present, are retained. The raw local response is
\begin{equation}
\label{eq:app_raw_local_response}
\bm{y}_{\star,i}^{\mathrm{raw}}
=
\mathcal{F}_{h,\star}^{\mathrm{raw}}
\left(
\mathbf{z}_{\star,i}
\right).
\end{equation}
When the zero-state response is finite, define the zero-state offset by
\begin{equation}
\label{eq:app_zero_state_offset}
\bm{b}_{\star,i}^{0}
=
\mathcal{F}_{h,\star}^{\mathrm{raw}}
\left(
\mathbf{z}_{\star,i}^{0}
\right).
\end{equation}
The calibrated local target is obtained by subtracting this offset:
\begin{equation}
\label{eq:app_calibrated_local_target}
\bm{y}_{\star,i}
=
\bm{y}_{\star,i}^{\mathrm{raw}}
-
\bm{b}_{\star,i}^{0}.
\end{equation}
Equivalently,
\begin{equation}
\label{eq:app_calibrated_local_operator}
\mathcal{F}_{h,\star}
\left(
\mathbf{z}_{\star,i}
\right)
=
\mathcal{F}_{h,\star}^{\mathrm{raw}}
\left(
\mathbf{z}_{\star,i}
\right)
-
\mathcal{F}_{h,\star}^{\mathrm{raw}}
\left(
\mathbf{z}_{\star,i}^{0}
\right).
\end{equation}
Substituting \(\mathbf{z}_{\star,i}=\mathbf{z}_{\star,i}^{0}\) into Eq.~\eqref{eq:app_calibrated_local_operator} gives
\begin{equation}
\label{eq:app_proof_calibrated_zero_consistency}
\mathcal{F}_{h,\star}
\left(
\mathbf{z}_{\star,i}^{0}
\right)
=
\mathcal{F}_{h,\star}^{\mathrm{raw}}
\left(
\mathbf{z}_{\star,i}^{0}
\right)
-
\mathcal{F}_{h,\star}^{\mathrm{raw}}
\left(
\mathbf{z}_{\star,i}^{0}
\right)
=
\bm{0}.
\end{equation}

The offset \(\bm{b}_{\star,i}^{0}\) can be interpreted as a background response at the zero physical field. In a homogeneous problem with translation equivariance, such a background response is translation invariant. To see this, let \(T_{\boldsymbol{\Delta}}\) denote a spatial translation and let \(\bm{Z}_{\star}^{0}\) be the global zero-state input. If
\begin{equation}
\label{eq:app_translation_equivariance_raw_operator}
\mathcal{F}_{h,\star}^{\mathrm{raw}}
\left[
T_{\boldsymbol{\Delta}}\bm{Z}_{\star}
\right]
=
T_{\boldsymbol{\Delta}}
\mathcal{F}_{h,\star}^{\mathrm{raw}}
\left[
\bm{Z}_{\star}
\right],
\end{equation}
and the zero-state input is translation invariant, namely
\begin{equation}
\label{eq:app_zero_state_global_translation_invariance}
T_{\boldsymbol{\Delta}}\bm{Z}_{\star}^{0}
=
\bm{Z}_{\star}^{0},
\end{equation}
then
\begin{equation}
\label{eq:app_zero_state_offset_translation_invariance}
\mathcal{F}_{h,\star}^{\mathrm{raw}}
\left[
\bm{Z}_{\star}^{0}
\right]
=
\mathcal{F}_{h,\star}^{\mathrm{raw}}
\left[
T_{\boldsymbol{\Delta}}\bm{Z}_{\star}^{0}
\right]
=
T_{\boldsymbol{\Delta}}
\mathcal{F}_{h,\star}^{\mathrm{raw}}
\left[
\bm{Z}_{\star}^{0}
\right].
\end{equation}
Thus the zero-state response is unchanged by translation. On a uniform grid, this means that the zero-state offset is the same at translated nodes,
\begin{equation}
\label{eq:app_zero_state_offset_constant}
\bm{b}_{\star,i+\boldsymbol{\Delta}}^{0}
=
\bm{b}_{\star,i}^{0}.
\end{equation}
It is therefore a translation invariant background drift rather than a response generated by local field variation. If coordinate-dependent sources, prescribed potentials, anisotropy, or direction-dependent forcing are present, the zero-state response may depend on these conditioning variables instead of being spatially constant. The pointwise calibration in Eq.~\eqref{eq:app_calibrated_local_operator} removes this prescribed offset before the local field-dependent operator is learned.

The LGNO local reconstruction then uses a bias-free stencil contraction:
\begin{equation}
\label{eq:app_lgno_local_reconstruction}
\widehat{\bm{y}}_{\star,i}
=
\widehat{\mathcal{F}}_{h,\star}^{\theta}
\left(
\mathbf{z}_{\star,i}
\right)
=
\sum_{j\in\mathcal{S}(i)}
\bm{G}_{\star,i,j}^{\theta}
\left(
\mathbf{z}_{\star,i}
\right)
\bm{u}_{i,j}.
\end{equation}
By definition of \(\mathbf{z}_{\star,i}^{0}\),
\begin{equation}
\label{eq:app_zero_state_primary_values}
\mathbf{z}_{\star,i}
=
\mathbf{z}_{\star,i}^{0}
\quad
\Longrightarrow
\quad
\bm{u}_{i,j}
=
\bm{0},
\qquad
j\in\mathcal{S}(i).
\end{equation}
Therefore,
\begin{equation}
\label{eq:app_proof_lgno_zero_consistency}
\widehat{\mathcal{F}}_{h,\star}^{\theta}
\left(
\mathbf{z}_{\star,i}^{0}
\right)
=
\sum_{j\in\mathcal{S}(i)}
\bm{G}_{\star,i,j}^{\theta}
\left(
\mathbf{z}_{\star,i}^{0}
\right)
\bm{u}_{i,j}
=
\bm{0}.
\end{equation}
This proves Eq.~\eqref{eq:zero_consistent_lgno_output}. Hence the zero-state output is removed first by calibration and then enforced by the reconstruction structure. Physically, the learned stencil cannot create a translation invariant background drift when no contracted primary-field values are present.

\paragraph{Error decomposition}
Assume that the calibrated compact-stencil operator admits the exact coefficient form
\begin{equation}
\label{eq:app_exact_local_coefficient_form}
\left(
\mathcal{F}_{h,\star}
[
\bm{Z}_{\star}
]
\right)_i
=
\mathcal{F}_{h,\star}
\left(
\mathbf{z}_{\star,i}
\right)
=
\sum_{j\in\mathcal{S}(i)}
\bm{G}_{\star,i,j}^{\ast}
\left(
\mathbf{z}_{\star,i}
\right)
\bm{u}_{i,j}.
\end{equation}
The assembled LGNO operator is
\begin{equation}
\label{eq:app_assembled_lgno_operator}
\left(
\widehat{\mathcal{F}}_{h,\star}^{\theta}
[
\bm{Z}_{\star}
]
\right)_i
=
\widehat{\mathcal{F}}_{h,\star}^{\theta}
\left(
\mathbf{z}_{\star,i}
\right)
=
\sum_{j\in\mathcal{S}(i)}
\bm{G}_{\star,i,j}^{\theta}
\left(
\mathbf{z}_{\star,i}
\right)
\bm{u}_{i,j}.
\end{equation}
For an admissible compact set of local windows \(\mathcal{K}_{\star}\), define
\begin{equation}
\label{eq:app_coefficient_approximation_error}
\varepsilon_{\theta}
=
\sup_{\mathbf{z}_{\star,i}\in\mathcal{K}_{\star}}
\max_{j\in\mathcal{S}(i)}
\left\|
\bm{G}_{\star,i,j}^{\theta}
\left(
\mathbf{z}_{\star,i}
\right)
-
\bm{G}_{\star,i,j}^{\ast}
\left(
\mathbf{z}_{\star,i}
\right)
\right\|_2 .
\end{equation}
The discrete norm is
\begin{equation}
\label{eq:app_discrete_l2_norm}
\|\bm{v}\|_{\ell_h^2}^2
=
h^d
\sum_i
|\bm{v}_i|^2.
\end{equation}
The total error satisfies
\begin{equation}
\label{eq:app_total_error_decomposition}
\begin{aligned}
R_h
\mathcal{F}_{\star}
[
\bm{Z}_{\star}
]
-
\widehat{\mathcal{F}}_{h,\star}^{\theta}
[
\bm{Z}_{\star}
]
=
\left(
R_h
\mathcal{F}_{\star}
[
\bm{Z}_{\star}
]
-
\mathcal{F}_{h,\star}
[
\bm{Z}_{\star}
]
\right)
+
\left(
\mathcal{F}_{h,\star}
[
\bm{Z}_{\star}
]
-
\widehat{\mathcal{F}}_{h,\star}^{\theta}
[
\bm{Z}_{\star}
]
\right).
\end{aligned}
\end{equation}
Using the assumed discretization estimate and the triangle inequality,
\begin{equation}
\label{eq:app_error_after_triangle}
\left\|
R_h
\mathcal{F}_{\star}
[
\bm{Z}_{\star}
]
-
\widehat{\mathcal{F}}_{h,\star}^{\theta}
[
\bm{Z}_{\star}
]
\right\|_{\ell_h^2}
\le
C_d h^p 
+
\left\|
\mathcal{F}_{h,\star}
[
\bm{Z}_{\star}
]
-
\widehat{\mathcal{F}}_{h,\star}^{\theta}
[
\bm{Z}_{\star}
]
\right\|_{\ell_h^2}.
\end{equation}
For each valid grid node, Eqs.~\eqref{eq:app_exact_local_coefficient_form} and \eqref{eq:app_assembled_lgno_operator} give
\begin{equation}
\label{eq:app_local_coefficient_error}
\left(
\mathcal{F}_{h,\star}
[
\bm{Z}_{\star}
]
-
\widehat{\mathcal{F}}_{h,\star}^{\theta}
[
\bm{Z}_{\star}
]
\right)_i
=
\sum_{j\in\mathcal{S}(i)}
\left(
\bm{G}_{\star,i,j}^{\ast}
\left(
\mathbf{z}_{\star,i}
\right)
-
\bm{G}_{\star,i,j}^{\theta}
\left(
\mathbf{z}_{\star,i}
\right)
\right)
\bm{u}_{i,j}.
\end{equation}
From Eq.~\eqref{eq:app_coefficient_approximation_error},
\begin{equation}
\label{eq:app_coefficient_error_bound}
\left\|
\bm{G}_{\star,i,j}^{\ast}
\left(
\mathbf{z}_{\star,i}
\right)
-
\bm{G}_{\star,i,j}^{\theta}
\left(
\mathbf{z}_{\star,i}
\right)
\right\|_2
\le
\varepsilon_{\theta}.
\end{equation}
Let
\begin{equation}
\label{eq:app_maximum_stencil_size}
s
=
\max_i
|\mathcal{S}(i)|.
\end{equation}
Then Cauchy--Schwarz gives
\begin{equation}
\label{eq:app_local_error_cauchy_schwarz}
\left|
\left(
\mathcal{F}_{h,\star}
[
\bm{Z}_{\star}
]
-
\widehat{\mathcal{F}}_{h,\star}^{\theta}
[
\bm{Z}_{\star}
]
\right)_i
\right|^2
\le
s
\varepsilon_{\theta}^2
\sum_{j\in\mathcal{S}(i)}
|\bm{u}_{i,j}|^2.
\end{equation}
Multiplying by \(h^d\), summing over all valid nodes, and using the uniformly bounded stencil overlap,
\begin{equation}
\label{eq:app_learned_coefficient_error_bound}
\left\|
\mathcal{F}_{h,\star}
[
\bm{Z}_{\star}
]
-
\widehat{\mathcal{F}}_{h,\star}^{\theta}
[
\bm{Z}_{\star}
]
\right\|_{\ell_h^2}
\le
C_{\star}
\varepsilon_{\theta}
\|\bm{u}\|_{\ell_h^2},
\end{equation}
where \(C_{\star}\) is independent of \(h\) and \(\theta\). Combining Eqs.~\eqref{eq:app_error_after_triangle} and \eqref{eq:app_learned_coefficient_error_bound} yields
\begin{equation}
\label{eq:app_final_error_decomposition}
\left\|
R_h
\mathcal{F}_{\star}
[
\bm{Z}_{\star}
]
-
\widehat{\mathcal{F}}_{h,\star}^{\theta}
[
\bm{Z}_{\star}
]
\right\|_{\ell_h^2}
\le
C_d h^p
+
C_{\star}
\varepsilon_{\theta}
\|\bm{u}\|_{\ell_h^2}.
\end{equation}
This proves Eq.~\eqref{eq:zero_consistent_error_decomposition}.

\paragraph{Rollout estimate}
For temporal evolution prediction, define
\begin{equation}
\label{eq:app_temporal_task_inputs}
\bm{Z}_{\mathrm{evo}}^n
=
(\bm{x},\bm{u}^n,\bm{f}^n),
\qquad
\widehat{\bm{Z}}_{\mathrm{evo}}^n
=
(\bm{x},\widehat{\bm{u}}^n,\bm{f}^n).
\end{equation}
The exact and learned one-step updates are
\begin{equation}
\label{eq:app_exact_one_step_update}
\bm{u}^{n+1}
=
\bm{u}^n
+
\Delta t
\mathcal{F}_{h,\mathrm{evo}}
(
\bm{u}^n,
\bm{f}^n
),
\end{equation}
and
\begin{equation}
\label{eq:app_lgno_one_step_update}
\widehat{\bm{u}}^{n+1}
=
\widehat{\bm{u}}^n
+
\Delta t
\widehat{\mathcal{F}}_{h,\mathrm{evo}}^{\theta}
[
\widehat{\bm{Z}}_{\mathrm{evo}}^n
].
\end{equation}
Let
\begin{equation}
\label{eq:app_rollout_error_definition}
\bm{e}^n
=
\bm{u}^n
-
\widehat{\bm{u}}^n.
\end{equation}
Assume that \(\mathcal{F}_{h,\mathrm{evo}}\) is Lipschitz continuous in the primary field for fixed source-related input:
\begin{equation}
\label{eq:app_lipschitz_assumption}
\left\|
\mathcal{F}_{h,\mathrm{evo}}
(
\bm{u},
\bm{f}
)
-
\mathcal{F}_{h,\mathrm{evo}}
(
\bm{v},
\bm{f}
)
\right\|_{\ell_h^2}
\le
L
\|\bm{u}-\bm{v}\|_{\ell_h^2}.
\end{equation}
Assume also the one-step approximation bound
\begin{equation}
\label{eq:app_one_step_error_bound}
\left\|
\mathcal{F}_{h,\mathrm{evo}}
(
\bm{u},
\bm{f}
)
-
\widehat{\mathcal{F}}_{h,\mathrm{evo}}^{\theta}
[
\bm{Z}_{\mathrm{evo}}
]
\right\|_{\ell_h^2}
\le
\zeta_{\theta}(h).
\end{equation}
Using Eqs.~\eqref{eq:app_exact_one_step_update}--\eqref{eq:app_rollout_error_definition},
\begin{equation}
\label{eq:app_error_recursion_expansion}
\begin{aligned}
\bm{e}^{n+1}
&=
\bm{e}^n
+
\Delta t
\left(
\mathcal{F}_{h,\mathrm{evo}}
(
\bm{u}^n,
\bm{f}^n
)
-
\widehat{\mathcal{F}}_{h,\mathrm{evo}}^{\theta}
[
\widehat{\bm{Z}}_{\mathrm{evo}}^n
]
\right) \\
&=
\bm{e}^n
+
\Delta t
\left(
\mathcal{F}_{h,\mathrm{evo}}
(
\bm{u}^n,
\bm{f}^n
)
-
\mathcal{F}_{h,\mathrm{evo}}
(
\widehat{\bm{u}}^n,
\bm{f}^n
)
\right) \\
&\quad
+
\Delta t
\left(
\mathcal{F}_{h,\mathrm{evo}}
(
\widehat{\bm{u}}^n,
\bm{f}^n
)
-
\widehat{\mathcal{F}}_{h,\mathrm{evo}}^{\theta}
[
\widehat{\bm{Z}}_{\mathrm{evo}}^n
]
\right).
\end{aligned}
\end{equation}
Taking the \(\ell_h^2\) norm and applying Eqs.~\eqref{eq:app_lipschitz_assumption} and \eqref{eq:app_one_step_error_bound},
\begin{equation}
\label{eq:app_error_recursion_bound}
\|\bm{e}^{n+1}\|_{\ell_h^2}
\le
(1+\Delta t L)
\|\bm{e}^{n}\|_{\ell_h^2}
+
\Delta t
\zeta_{\theta}(h).
\end{equation}
Iteration gives
\begin{equation}
\label{eq:app_iterated_error_recursion}
\|\bm{e}^{n}\|_{\ell_h^2}
\le
(1+\Delta t L)^n
\|\bm{e}^{0}\|_{\ell_h^2}
+
\Delta t
\zeta_{\theta}(h)
\sum_{m=0}^{n-1}
(1+\Delta t L)^m.
\end{equation}
For \(L>0\),
\begin{equation}
\label{eq:app_geometric_sum}
\sum_{m=0}^{n-1}
(1+\Delta t L)^m
=
\frac{(1+\Delta t L)^n-1}{\Delta t L}.
\end{equation}
Since \(T=n\Delta t\),
\begin{equation}
\label{eq:app_exponential_bound}
(1+\Delta t L)^n
\le
e^{LT}.
\end{equation}
Combining Eqs.~\eqref{eq:app_iterated_error_recursion}--\eqref{eq:app_exponential_bound},
\begin{equation}
\label{eq:app_final_rollout_error_bound}
\|\bm{e}^{n}\|_{\ell_h^2}
\le
e^{LT}
\|\bm{e}^{0}\|_{\ell_h^2}
+
\frac{e^{LT}-1}{L}
\zeta_{\theta}(h),
\qquad
T=n\Delta t.
\end{equation}
This proves Eq.~\eqref{eq:zero_consistent_rollout_error}.

When \(L=0\), Eq.~\eqref{eq:app_error_recursion_bound} reduces to
\begin{equation}
\label{eq:app_zero_lipschitz_recursion}
\|\bm{e}^{n+1}\|_{\ell_h^2}
\le
\|\bm{e}^{n}\|_{\ell_h^2}
+
\Delta t
\zeta_{\theta}(h).
\end{equation}
Thus,
\begin{equation}
\label{eq:app_zero_lipschitz_final_bound}
\|\bm{e}^{n}\|_{\ell_h^2}
\le
\|\bm{e}^{0}\|_{\ell_h^2}
+
T
\zeta_{\theta}(h).
\end{equation}
The second term in Eq.~\eqref{eq:zero_consistent_rollout_error} is therefore interpreted as \(T\zeta_{\theta}(h)\).

The rollout estimate is conditional on a bounded admissible trajectory, the Lipschitz bound in Eq.~\eqref{eq:app_lipschitz_assumption}, and the one-step bound in Eq.~\eqref{eq:app_one_step_error_bound}. Together with Eq.~\eqref{eq:zero_consistent_error_decomposition}, it links the autoregressive error to the compact-stencil discretization error \(h^p\) and the learned coefficient error \(\varepsilon_{\theta}\).

\section{Derivation of symmetry-induced network folding}
\label{sec:symmetry_folding}

This appendix derives Eqs.~\eqref{eq:symmetry_orbit_augmentation}--\eqref{eq:symmetry_folded_reconstruction} and clarifies the assumptions under which componentwise folding is valid. Let
\begin{equation}
\label{eq:app_symmetry_valid_pair}
\bm{Y}_{\star}
=
\mathcal{F}_{\star}
\left[
\bm{Z}_{\star}
\right]
\end{equation}
be a calibrated input-target pair. Here \(\mathcal{F}_{\star}\) denotes the exact calibrated task operator, whereas \(\widehat{\mathcal{F}}_{\star}^{\theta}\) denotes the learned LGNO approximation. Therefore, quantities generated by \(\mathcal{F}_{\star}\) are exact targets and are not marked with a hat. Quantities generated by \(\widehat{\mathcal{F}}_{\star}^{\theta}\) are model predictions and are marked with a hat.

Assume that the calibrated task operator is \(\Gamma\)-equivariant:
\begin{equation}
\label{eq:app_symmetry_equivariance_repeat}
\mathcal{F}_{\star}
\left[
\rho_{\mathrm{in}}(\gamma)\bm{Z}_{\star}
\right]
=
\rho_{\mathrm{out}}(\gamma)
\mathcal{F}_{\star}
\left[
\bm{Z}_{\star}
\right],
\qquad
\forall \gamma\in\Gamma .
\end{equation}
Substituting Eq.~\eqref{eq:app_symmetry_valid_pair} into Eq.~\eqref{eq:app_symmetry_equivariance_repeat} gives
\begin{equation}
\label{eq:app_symmetry_transformed_output}
\mathcal{F}_{\star}
\left[
\rho_{\mathrm{in}}(\gamma)\bm{Z}_{\star}
\right]
=
\rho_{\mathrm{out}}(\gamma)
\bm{Y}_{\star}.
\end{equation}
Hence
\begin{equation}
\label{eq:app_symmetry_transformed_pair}
\left(
\rho_{\mathrm{in}}(\gamma)\bm{Z}_{\star},
\rho_{\mathrm{out}}(\gamma)\bm{Y}_{\star}
\right)
\end{equation}
is also generated by the same calibrated operator. Varying \(\gamma\in\Gamma\) gives the orbit of valid input-target pairs in Eq.~\eqref{eq:symmetry_orbit_augmentation}. If \(\Gamma\) is continuous, this orbit is usually used in practice by sampling a finite number of group actions.

We next derive the componentwise folding relation. Let \(\mathcal{Q}\) be a set of target components, field channels, or coordinate directions lying in the same \(\Gamma\)-orbit. Fix a canonical component \(q_0\in\mathcal{Q}\). For each \(q\in\mathcal{Q}\), assume that there exists a transformation \(\gamma_q\in\Gamma\) that maps the \(q\)-component to the canonical component \(q_0\). This assumption is required for folding; if no such canonical map exists, the components should not be folded.

Let \(\mathcal{E}_i\) denote local window extraction:
\begin{equation}
\label{eq:app_local_window_extraction}
\mathbf{z}_{\star,i}
=
\mathcal{E}_i
\left[
\bm{Z}_{\star}
\right].
\end{equation}
We assume that the transformations used for folding are admissible for the grid and boundary treatment, so that applying the group action and then extracting a local window is well defined. For transformations that do not map grid nodes exactly to grid nodes, the representation \(\rho_{\mathrm{in}}\) is understood to include the interpolation or resampling used by the discretization.

Applying \(\gamma_q\) to the input field and then extracting the local window gives
\begin{equation}
\label{eq:app_transformed_local_window}
\mathbf{z}_{\star,i}^{(\gamma_q)}
=
\mathcal{E}_i
\left[
\rho_{\mathrm{in}}(\gamma_q)\bm{Z}_{\star}
\right].
\end{equation}
The corresponding transformed target follows from Eq.~\eqref{eq:app_symmetry_transformed_output}:
\begin{equation}
\label{eq:app_transformed_local_target}
\bm{y}_{\star,i}^{(\gamma_q)}
=
\left(
\rho_{\mathrm{out}}(\gamma_q)\bm{Y}_{\star}
\right)_i
=
\left(
\mathcal{F}_{\star}
\left[
\rho_{\mathrm{in}}(\gamma_q)\bm{Z}_{\star}
\right]
\right)_i .
\end{equation}
Thus the transformed target is still an exact target generated by the same calibrated operator, so it does not carry a hat.

Let \(P_{q_0}^{\mathrm{in}}\) and \(P_{q_0}^{\mathrm{out}}\) denote the projections onto the canonical input and target components after the symmetry action. The folded local input and target are defined by
\begin{equation}
\label{eq:app_folded_local_sample}
\widetilde{\mathbf{z}}_{\star,i,q}
=
P_{q_0}^{\mathrm{in}}
\mathbf{z}_{\star,i}^{(\gamma_q)}
=
P_{q_0}^{\mathrm{in}}
\mathcal{E}_i
\left[
\rho_{\mathrm{in}}(\gamma_q)\bm{Z}_{\star}
\right],
\qquad
\widehat{\widetilde{\mathbf{y}}}_{\star,i,q}
=
P_{q_0}^{\mathrm{out}}
\bm{y}_{\star,i}^{(\gamma_q)}
=
P_{q_0}^{\mathrm{out}}
\left(
\rho_{\mathrm{out}}(\gamma_q)\bm{Y}_{\star}
\right)_i .
\end{equation}
This gives Eq.~\eqref{eq:symmetry_component_folding}. Combining Eqs.~\eqref{eq:app_transformed_local_target} and \eqref{eq:app_folded_local_sample} gives
\begin{equation}
\label{eq:app_folded_target_from_same_operator}
\widehat{\widetilde{\mathbf{y}}}_{\star,i,q}
=
P_{q_0}^{\mathrm{out}}
\left(
\mathcal{F}_{\star}
\left[
\rho_{\mathrm{in}}(\gamma_q)\bm{Z}_{\star}
\right]
\right)_i .
\end{equation}
Therefore, the folded target is the canonical component of the exact output produced by the same calibrated operator acting on the transformed input. By locality, this canonical output component depends only on the corresponding folded local window and the prescribed conditioning channels. Hence all folded samples can be treated as samples of the same canonical local learning problem.

It remains to define the primary-field values contracted in the LGNO stencil. Since the primary field \(\bm{u}\) is contained in \(\bm{Z}_{\star}\), the group action on \(\bm{u}\) is induced by \(\rho_{\mathrm{in}}\). Denote this induced action by \(\rho_{\mathrm{u}}\), and let \(\mathcal{E}_i^{\mathrm{u}}\) denote the primary-field part of the local window extraction. For \(j\in\mathcal{S}(i)\), define
\begin{equation}
\label{eq:app_folded_primary_field}
\widetilde{\bm{u}}_{i,j,q}
=
\left[
P_{q_0}^{\mathrm{u}}
\mathcal{E}_i^{\mathrm{u}}
\left[
\rho_{\mathrm{u}}(\gamma_q)\bm{u}
\right]
\right]_j,
\qquad
j\in\mathcal{S}(i),
\end{equation}
where \(P_{q_0}^{\mathrm{u}}\) selects the canonical primary-field channels after the symmetry action.

Under the compact-stencil representation of the calibrated local operator, there exists an exact canonical coefficient rule \(\bm{G}_{\star,i,j}^{\ast}\) such that
\begin{equation}
\label{eq:app_exact_folded_stencil_form}
\widehat{\widetilde{\mathbf{y}}}_{\star,i,q}
=
\sum_{j\in\mathcal{S}(i)}
\bm{G}_{\star,i,j}^{\ast}
\left(
\widetilde{\mathbf{z}}_{\star,i,q}
\right)
\widetilde{\bm{u}}_{i,j,q}.
\end{equation}
Here \(\widehat{\widetilde{\mathbf{y}}}_{\star,i,q}\) is the exact folded target, so it is not marked with a hat.

LGNO replaces the exact coefficient rule by a learned coefficient generator:
\begin{equation}
\label{eq:app_folded_shared_generator}
\Phi_{\boldsymbol{\theta}}
:
\widetilde{\mathbf{z}}_{\star,i,q}
\mapsto
\left\{
\bm{G}_{\star,i,j}^{\theta}
\left(
\widetilde{\mathbf{z}}_{\star,i,q}
\right)
\right\}_{j\in\mathcal{S}(i)}.
\end{equation}
The model prediction for the folded local target is therefore
\begin{equation}
\label{eq:app_folded_local_reconstruction}
\widehat{\widehat{\widetilde{\mathbf{y}}}}_{\star,i,q}
=
\sum_{j\in\mathcal{S}(i)}
\bm{G}_{\star,i,j}^{\theta}
\left(
\widetilde{\mathbf{z}}_{\star,i,q}
\right)
\widetilde{\bm{u}}_{i,j,q},
\qquad
q\in\mathcal{Q}.
\end{equation}
This gives Eq.~\eqref{eq:symmetry_folded_reconstruction}. The hat on \(\widehat{\widehat{\widetilde{\mathbf{y}}}}_{\star,i,q}\) indicates that this quantity is produced by the learned LGNO reconstruction, not by the exact calibrated operator.

The construction requires three conditions. First, the calibrated task operator must satisfy the equivariance relation in Eq.~\eqref{eq:app_symmetry_equivariance_repeat}. Second, the components in \(\mathcal{Q}\) must be related by admissible group actions so that a canonical component \(q_0\) is well defined. Third, the local stencil and boundary treatment must be compatible with the transformations used for folding. If zero-point calibration or residualization is used, the calibrated target must still transform under \(\rho_{\mathrm{out}}\). If symmetry-breaking sources, potentials, anisotropy, or direction-dependent forcing are present and are not included in an equivariant input representation, the folded samples no longer correspond to the same canonical local operator law.

\section{Detailed benchmark prediction results of selected models}
\label{Appendix:all_results}

This appendix reports the benchmark prediction results for the four representative models used in the focused comparison: MLPConv, DeepONet, LOINN, and LGNO. These methods cover a local stencil map, a global branch--trunk operator map, a few-shot local operator-learning model, and the proposed differential-stencil reconstruction. For conciseness, the main paper reports a compact LGNO-only summary for the principal benchmark cases analyzed in the main discussion, as shown in Table~\ref{tab:concise_results}. The detailed tables below retain all available configurations of these four models across the 1D, 2D, and 3D benchmarks. For each case, we report the number of trainable parameters, the final training loss, the test relative \(L^2\) error, and the error ratio with respect to the best-performing model in the same benchmark.

All selected models were evaluated using the same unified experimental scripts and data-processing pipeline across the different benchmark problems. No benchmark-specific redesign or extensive case-by-case hyperparameter tuning was performed for the retained baseline methods. The comparison is therefore intended to reflect model behavior under a common low-data protocol rather than the fully optimized limit of each architecture.

\begin{table*}[htbp]
\centering
\caption{Benchmark prediction results of selected models.}
\label{tab:benchmark_prediction_results_selected}
\setlength{\tabcolsep}{3.5pt}
\renewcommand{\arraystretch}{1}
\begin{adjustbox}{max width=\textwidth}
\begin{tabular}{l l l l l l l}
\toprule
Benchmark & Method & Hidden & Parameters & Training loss & Test rel. \(L^2\) & Error ratio \\
\midrule
\multirow{8}{*}{1D Diffusion}
& \textbf{LGNO} & \textbf{2} & \textbf{\num{23}} & \textbf{\num{5.891e-04}} & \textbf{0.012} & \textbf{1.0\(\times\)} \\
& LGNO & 4 & \num{51} & \num{2.395e-03} & 0.025 & 2.0\(\times\) \\
& MLPConv & 4 & \num{41} & 0.014 & 0.053 & 4.4\(\times\) \\
& LOINN & 4 & \num{85} & 0.126 & 0.857 & 70.9\(\times\) \\
& MLPConv & 2 & \num{29} & 27.306 & 0.947 & 78.3\(\times\) \\
& LOINN & 2 & \num{35} & 74.060 & 3.068 & 253.7\(\times\) \\
& DeepONet & 4 & \num{221} & 165.298 & 5.615 & 464.3\(\times\) \\
& DeepONet & 2 & \num{95} & 178.280 & 23.492 & 1942.6\(\times\) \\
\midrule
\multirow{5}{*}{2D Burgers}
& \textbf{LGNO} & \textbf{4} & \textbf{\num{105}} & \textbf{\num{5.076e-03}} & \textbf{0.050} & \textbf{1.0\(\times\)} \\
& LGNO & 8 & \num{233} & \num{6.446e-03} & 0.093 & 1.9\(\times\) \\
& MLPConv & 4 & \num{65} & \num{3.754e-03} & 0.380 & 7.7\(\times\) \\
& LOINN & 4 & \num{85} & 0.020 & 0.463 & 9.4\(\times\) \\
& DeepONet & 8 & \num{175585} & 0.099 & 1.003 & 20.3\(\times\) \\
\midrule
\multirow{5}{*}{2D Navier--Stokes}
& \textbf{LGNO} & \textbf{16} & \textbf{\num{1188}} & \textbf{\num{3.371e-07}} & \textbf{0.008} & \textbf{1.0\(\times\)} \\
& LGNO & 8 & \num{548} & \num{8.170e-07} & 0.010 & 1.3\(\times\) \\
& MLPConv & 16 & \num{610} & \num{2.882e-06} & 0.017 & 2.1\(\times\) \\
& DeepONet & 16 & \num{34738} & \num{4.491e-03} & 1.014 & 124.7\(\times\) \\
& LOINN & 16 & \num{1170} & \num{3.135e-08} & -- & -- \\
\midrule
\multirow{6}{*}{\makecell[l]{2D Schr\"odinger\\(Gross--Pitaevskii)}}
& \textbf{LGNO(folded)} & \textbf{12} & \textbf{\num{370}} & \textbf{\num{2.510e-05}} & \textbf{0.116} & \textbf{1.0\(\times\)} \\
& LGNO(folded) & 24 & \num{1018} & \num{2.393e-04} & 0.283 & 2.4\(\times\) \\
& LGNO(unfolded) & 12 & \num{500} & \num{2.836e-05} & 0.379 & 3.3\(\times\) \\
& DeepONet & 12 & \num{148598} & 1.182 & 1.067 & 9.2\(\times\) \\
& LOINN & 12 & \num{674} & 1.217 & 3.085 & 26.6\(\times\) \\
& MLPConv & 12 & \num{518} & 0.592 & 3.548 & 30.6\(\times\) \\
\midrule
\multirow{4}{*}{\makecell[l]{3D Schr\"odinger\\(Perturbed Harmonic)}}
& \textbf{LGNO} & \textbf{16} & \textbf{\num{988}} & \textbf{\num{9.790e-10}} & \textbf{0.003} & \textbf{1.0\(\times\)} \\
& MLPConv & 16 & \num{1618} & \num{9.563e-07} & 0.019 & 6.2\(\times\) \\
& LOINN & 16 & \num{1890} & \num{7.975e-07} & 0.027 & 8.7\(\times\) \\
& DeepONet & 16 & \num{198594} & \num{8.447e-03} & 1.122 & 358.4\(\times\) \\
\bottomrule
\end{tabular}
\end{adjustbox}
\end{table*}

Against these baselines, LGNO uses a local coefficient generator together with a differential-stencil reconstruction. This restricted comparison helps isolate how direct operator mapping and differential local reconstruction behave when only one trajectory or a small number of samples is available. The reported results should therefore be interpreted under this unified experimental setting, while the main conclusions focus on the observed accuracy parameter trade-off and local generalization behavior of LGNO.

\section{Additional results for the 2D Navier--Stokes system} \label{Appendix:2DNS}

As a supplement to Paragraph~\ref{Para:2DNS}, where the concise comparison between LGNO-16 and MLPConv-16 has been presented, Fig.~\ref{2DNS_Appendix} provides more detailed visual comparisons for the 2D Navier--Stokes example. The baseline models such as MLPConv and DeepONet are able to reproduce the main flow structures at early rollout times, but their predictions become progressively less accurate as the rollout proceeds. In later snapshots, visible errors appear in both velocity components, including blurred local structures, weakened vortical features, and misplaced flow patterns. These additional visual results complement the quantitative comparison reported in the main text and further illustrate the long-time rollout advantage of LGNO.

\begin{figure}[!htbp]
	\centering
	\includegraphics[width=\linewidth]{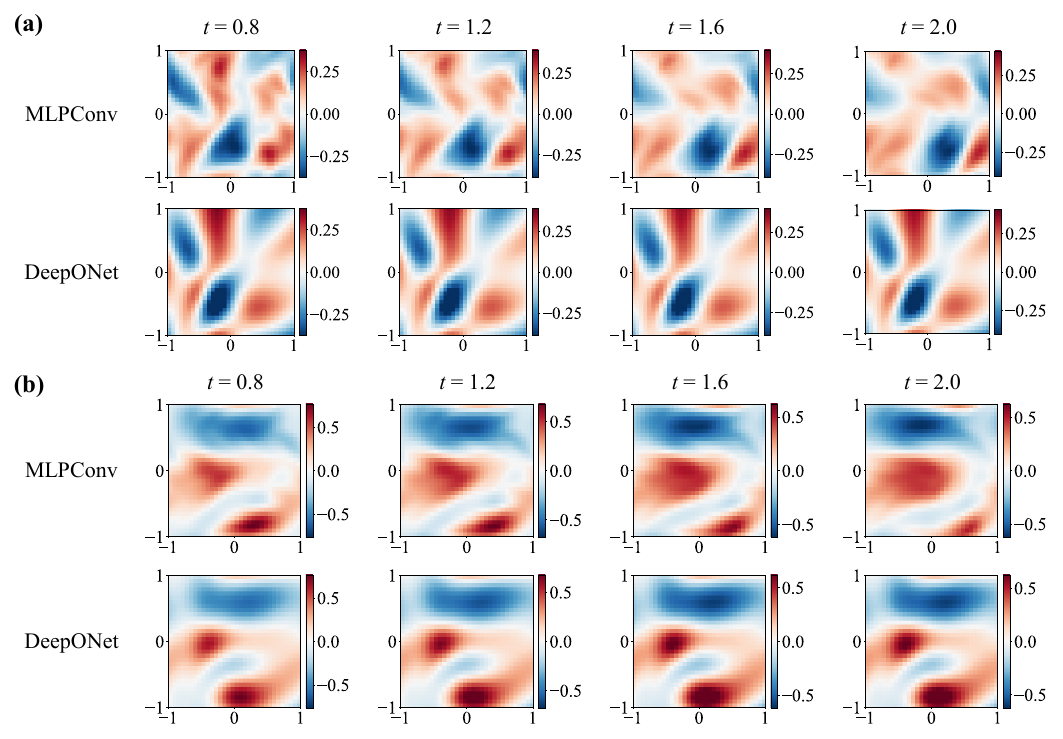}
    \caption{
    Baseline predictions for the 2D Navier--Stokes system.
    Rows correspond to MLPConv and DeepONet.
    Columns show snapshots at $t=0.8,1.2,1.6,2.0$.
    (a) Horizontal velocity component $u$.
    (b) Vertical velocity component $v$.
    }
    \label{2DNS_Appendix}
\end{figure}

\section{Additional results for the 2D Burgers equation} \label{Appendix:2D Burgers}

As a supplement to the concise comparison presented in Section~\ref{2D Burgers}, Fig.~\ref{[Appendix]2DBurger} provides additional qualitative comparisons on the 2D Burgers test set. In these examples, LGNO-8 remains closer to the reference solution, especially in regions with sharper spatial variations.

\begin{figure}[!htbp]
	\centering
\includegraphics[width=\linewidth]{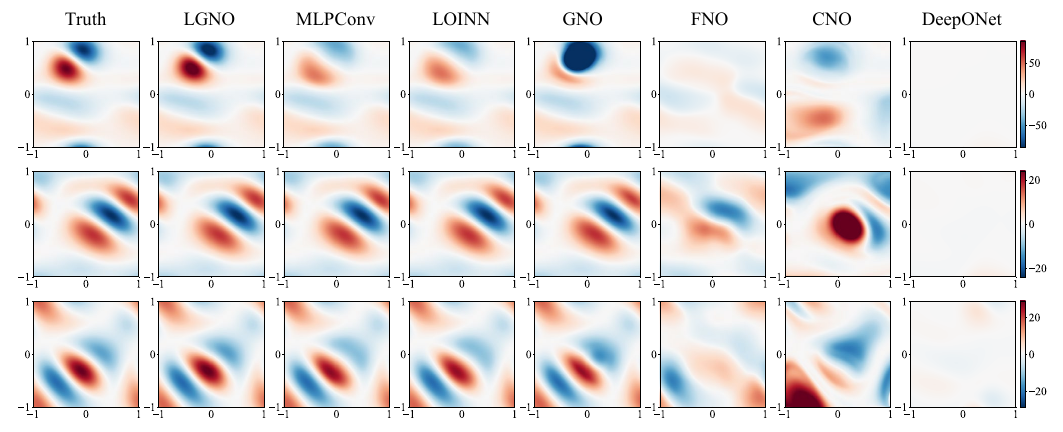}
    \caption{
    Additional test examples for the 2D Burgers equation.
    From left to right: reference solution, LGNO-8, MLPConv-8, LOINN-8, GNO-8, FNO-8, CNO-8,
    and DeepONet-8.
    }
    \label{[Appendix]2DBurger}
\end{figure}

\section{Dataset examples for the 2D Gross--Pitaevskii equation} \label{Appendix:2DGP}

\begin{figure}[!htbp]
	\centering
\includegraphics[width=0.96\linewidth]{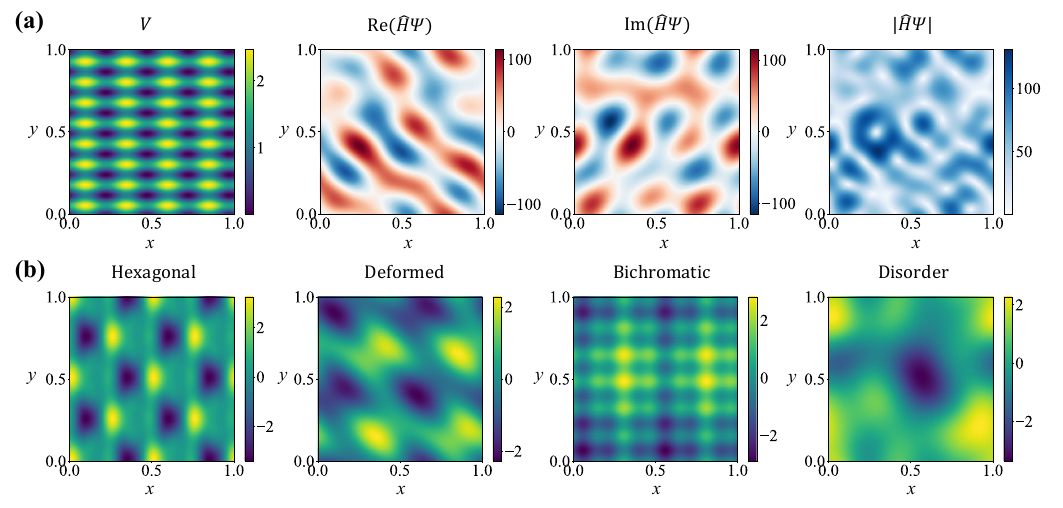}
	\caption{
    Dataset for the 2D Gross--Pitaevskii equation.
    (a) Training sample generated from an optical lattice potential.
    From left to right: potential field $V$, real part
    $\mathrm{Re}(\hat{H}\psi)$, imaginary part
    $\mathrm{Im}(\hat{H}\psi)$, and magnitude $|\hat{H}\psi|$.
    (b) Test potentials used in the experiments: hexagonal, deformed,
    bichromatic, and disorder potentials.
    }
    \label{fig:appendix_2dgp_dataset}
\end{figure}

As a supplement to Section~\ref{paragraph:2DGP}, \figref{fig:appendix_2dgp_dataset} illustrates the dataset construction for the 2D Gross--Pitaevskii equation. 
The training data are generated with an optical lattice potential, while the test set uses four different periodic potential types, namely hexagonal, deformed, bichromatic, and disorder potentials. This train--test split is designed to evaluate how the learned operator generalizes to potential fields that are not observed during training.

\section{Robustness results}
\label{app:robustness}

As discussed in Section~\ref{sec:robustness}, the robustness test evaluates the sensitivity of LGNO-8 and MLPConv-8 to multiplicative input noise. Table~\ref{tab:noise_casewise_with_summary} provides the complete case-wise results for the five evaluation seeds \(447\)--\(451\), together with the average relative \(L^2\) error at each noise level.

\begin{table}[!htbp]
\centering
\caption{Case-wise comparison of relative \(L^2\) errors under different noise
levels. Case 1--5 correspond to evaluation seeds \(447-451\),
respectively. For each case, the lower error between LGNO-8 and MLPConv-8 is
highlighted in blue. The column ``Avg.'' reports the mean error over the five
cases at each noise level. The last column reports the relative advantage
computed from the mean error, \((e_{\mathrm{MLPConv}}-e_{\mathrm{LGNO}})/
e_{\mathrm{MLPConv}}\). Positive values indicate lower average error for
LGNO-8.}
\label{tab:noise_casewise_with_summary}
\begin{adjustbox}{max width=\linewidth}
\begin{tabular}{c|c|ccccc|c|c}
\toprule
\(\alpha\) & Method & Case 1 & Case 2 & Case 3 & Case 4 & Case 5 & Avg. & Advantage \\
\midrule

\multirow{2}{*}{0.00}
& LGNO-8
& \cellcolor{blue!5}0.118 & \cellcolor{blue!5}0.042 & \cellcolor{blue!5}0.106 & \cellcolor{blue!5}0.128 & \cellcolor{blue!5}0.112
& \cellcolor{blue!5}0.101 & \multirow{2}{*}{\(+71\%\)} \\
& MLPConv-8
& 0.454 & 0.069 & 0.268 & 0.557 & 0.382
& 0.346 & \\

\midrule

\multirow{2}{*}{0.01}
& LGNO-8
& \cellcolor{blue!5}0.105 & \cellcolor{blue!5}0.045 & \cellcolor{blue!5}0.094 & \cellcolor{blue!5}0.112 & \cellcolor{blue!5}0.099
& \cellcolor{blue!5}0.091 & \multirow{2}{*}{\(+74\%\)} \\
& MLPConv-8
& 0.453 & 0.073 & 0.266 & 0.557 & 0.381
& 0.346 & \\

\midrule

\multirow{2}{*}{0.02}
& LGNO-8
& \cellcolor{blue!5}0.114 & \cellcolor{blue!5}0.056 & \cellcolor{blue!5}0.106 & \cellcolor{blue!5}0.119 & \cellcolor{blue!5}0.108
& \cellcolor{blue!5}0.100 & \multirow{2}{*}{\(+71\%\)} \\
& MLPConv-8
& 0.449 & 0.078 & 0.265 & 0.553 & 0.377
& 0.345 & \\

\midrule

\multirow{2}{*}{0.03}
& LGNO-8
& \cellcolor{blue!5}0.296 & \cellcolor{blue!5}0.068 & \cellcolor{blue!5}0.190 & \cellcolor{blue!5}0.439 & \cellcolor{blue!5}0.250
& \cellcolor{blue!5}0.249 & \multirow{2}{*}{\(+30\%\)} \\
& MLPConv-8
& 0.460 & 0.089 & 0.280 & 0.562 & 0.388
& 0.356 & \\

\midrule

\multirow{2}{*}{0.04}
& LGNO-8
& \cellcolor{blue!5}0.090 & \cellcolor{blue!5}0.087 & \cellcolor{blue!5}0.093 & \cellcolor{blue!5}0.155 & \cellcolor{blue!5}0.081
& \cellcolor{blue!5}0.101 & \multirow{2}{*}{\(+72\%\)} \\
& MLPConv-8
& 0.457 & 0.101 & 0.279 & 0.560 & 0.386
& 0.356 & \\

\midrule

\multirow{2}{*}{0.05}
& LGNO-8
& \cellcolor{blue!5}0.166 & \cellcolor{blue!5}0.103 & \cellcolor{blue!5}0.158 & \cellcolor{blue!5}0.186 & \cellcolor{blue!5}0.156
& \cellcolor{blue!5}0.154 & \multirow{2}{*}{\(+57\%\)} \\
& MLPConv-8
& 0.458 & 0.118 & 0.285 & 0.560 & 0.388
& 0.361 & \\

\midrule

\multirow{2}{*}{0.06}
& LGNO-8
& \cellcolor{blue!5}0.108 & \cellcolor{blue!5}0.122 & \cellcolor{blue!5}0.139 & \cellcolor{blue!5}0.103 & \cellcolor{blue!5}0.111
& \cellcolor{blue!5}0.117 & \multirow{2}{*}{\(+68\%\)} \\
& MLPConv-8
& 0.456 & 0.131 & 0.289 & 0.558 & 0.386
& 0.364 & \\

\midrule

\multirow{2}{*}{0.07}
& LGNO-8
& \cellcolor{blue!5}0.115 & \cellcolor{blue!5}0.138 & \cellcolor{blue!5}0.158 & \cellcolor{blue!5}0.107 & \cellcolor{blue!5}0.122
& \cellcolor{blue!5}0.128 & \multirow{2}{*}{\(+66\%\)} \\
& MLPConv-8
& 0.461 & 0.146 & 0.301 & 0.561 & 0.393
& 0.372 & \\

\midrule

\multirow{2}{*}{0.08}
& LGNO-8
& \cellcolor{blue!5}0.128 & \cellcolor{blue!5}0.155 & \cellcolor{blue!5}0.167 & \cellcolor{blue!5}0.171 & \cellcolor{blue!5}0.128
& \cellcolor{blue!5}0.150 & \multirow{2}{*}{\(+60\%\)} \\
& MLPConv-8
& 0.462 & 0.162 & 0.309 & 0.561 & 0.395
& 0.378 & \\

\midrule

\multirow{2}{*}{0.09}
& LGNO-8
& \cellcolor{blue!5}0.168 & \cellcolor{blue!5}0.172 & \cellcolor{blue!5}0.208 & \cellcolor{blue!5}0.174 & \cellcolor{blue!5}0.170
& \cellcolor{blue!5}0.178 & \multirow{2}{*}{\(+52\%\)} \\
& MLPConv-8
& 0.451 & 0.178 & 0.309 & 0.552 & 0.386
& 0.375 & \\

\midrule

\multirow{2}{*}{0.10}
& LGNO-8
& \cellcolor{blue!5}0.276 & 0.346 & 0.366 & \cellcolor{blue!5}0.285 & \cellcolor{blue!5}0.310
& \cellcolor{blue!5}0.317 & \multirow{2}{*}{\(+18\%\)} \\
& MLPConv-8
& 0.459 & \cellcolor{blue!5}0.193 & \cellcolor{blue!5}0.322 & 0.557 & 0.395
& 0.385 & \\

\bottomrule
\end{tabular}
\end{adjustbox}
\end{table}

Overall, LGNO-8 achieves a lower mean error than MLPConv-8 for all tested noise amplitudes \(\alpha\in[0,0.10]\). In the low-noise regime, \(\alpha\leq 0.02\), LGNO-8 shows a large average advantage of about \(71\%\)--\(74\%\), indicating that the local gradient-based operator representation remains substantially more accurate when the perturbation is weak. As the noise level increases, the advantage gradually decreases but remains positive. For \(\alpha=0.10\), MLPConv-8 performs better in two individual cases, while LGNO-8 still attains the smaller average error, with an \(18\%\) relative advantage. These results support the observation in Section~\ref{sec:robustness} that LGNO preserves stronger average robustness under moderate input perturbations, although the margin becomes smaller as the noise amplitude increases.

\FloatBarrier
\bibliography{refs}
 
\appendix

\end{document}